\documentclass[12pt]{article}
\usepackage{inputenc}

\usepackage{amsmath,amssymb,amsfonts,mathbbol}
\usepackage{thmtools,thm-restate}
\usepackage[numbers,sort,square]{natbib}
\usepackage[margin=1.0in]{geometry}
\usepackage{amsmath,amssymb,amsthm}
\usepackage{hyperref}
\usepackage{color}
\definecolor{darkblue}{rgb}{0.0,0.0,0.3}
\definecolor{darkgreen}{rgb}{0.0,0.3,0.0}
\hypersetup{colorlinks,breaklinks,
	linkcolor=darkblue,urlcolor=darkblue,
	anchorcolor=darkblue,citecolor=darkblue}
\usepackage[colorinlistoftodos,textsize=tiny]{todonotes} 

\newcommand{\Comments}{0}
\newcommand{\mynote}[2]{\ifnum\Comments=1\textcolor{#1}{#2}\fi}
\newcommand{\mytodo}[2]{\ifnum\Comments=1	\todo[linecolor=#1!80!black,backgroundcolor=#1,bordercolor=#1!80!black]{#2}\fi}

\ifnum\Comments=1               
\fi

\usepackage{float}
\newcommand{\Ind}[1]{\mathbf{1}\{#1\}}
\newcommand{\toto}{\rightrightarrows}

\renewcommand{\i}{{(i)}}

\newcommand{\E}{\mathbb{E}}

\renewcommand{\H}{\mathcal{H}}
\renewcommand{\L}{\mathbf{L}}
\newcommand{\R}{\mathcal{R}}
\newcommand{\reals}{\mathbb{R}}
\newcommand{\Sc}{\mathcal{S}}
\newcommand{\T}{\mathcal{T}}

\newcommand{\X}{\mathcal{X}}
\newcommand{\Y}{\mathcal{Y}}
\newcommand{\simplex}{\Delta_\Y}
\renewcommand{\i}{{(i)}}
\newcommand{\prop}[1]{\mathrm{prop}[#1]}

\newcommand{\Gamext}{{\hat{\Gamma}}}
\newcommand{\Lreg}{L^{\R,\lambda}}
\newcommand{\ones}{\mathbb{1}}
\newcommand{\relint}{\mathrm{relint}}
\newcommand{\equivs}{\equiv_{\vec s}}
\newcommand{\ND}{\mathbf{nondiff}}
\newcommand{\risk}[1]{\underline{#1}}

\newcommand{\argmax}{\arg\,\max}
\newcommand{\argmin}{\arg\,\min}
\renewcommand{\vec}[1]{\mathbf{#1}}

\newtheorem{definition}{Definition}

\newtheorem{theorem}{Theorem}
\newtheorem{proposition}{Proposition}
\newtheorem{corollary}{Corollary}

\usepackage{authblk}
\title{Using Property Elicitation to Understand the Impacts of Fairness Regularizers}
\author[1]{Jessie Finocchiaro} 
\affil[1]{Center for Research on Computation and Society, Harvard University, \texttt{jessie@seas.harvard.edu}}
\date{}

\begin{document}

\maketitle

\begin{abstract}
Predictive algorithms are often trained by optimizing some loss function, to which regularization functions are added to impose a penalty for violating constraints.
As expected, the addition of such regularization functions can change the minimizer of the objective.
It is not well-understood which regularizers change the minimizer of the loss, and, when the minimizer does change, \emph{how} it changes.
We use \emph{property elicitation} to take first steps towards understanding the joint relationship between the loss, regularization functions, and the optimal decision for a given problem instance. 
In particular, we give a necessary and sufficient condition on loss and regularizer pairs for when a property changes with the addition of the regularizer, and examine some commonly used regularizers satisfying this condition from the fair machine learning literature.
We empirically demonstrate how algorithmic decision-making changes as a function of both data distribution changes and hardness of the constraints.
\end{abstract}

\section{Introduction}
Machine learning is increasingly being used for prediction and resource allocation tasks pertaining to human livelihood; algorithms often make predictions based on patterns in historical data to make or supplement decisions about future events.
For example, algorithms are commonly used to determine a whether or not a loan applicant should receive a loan~\citep{sheikh2020approach,arutjothi2017prediction,singh2021prediction}, estimate a patient's risk of heart disease~\citep{heart,lloyd2010cardiovascular,goldstein2017moving}, and estimate need for public assistance~\citep{kube2023community}, among other settings.
Typically, an algorithm tries to predict something like the probability of an applicant repaying the loan if granted one, and then uses this prediction to assign a treatment to the applicant, such as granting or not granting a loan.
Implicit in this model is the use of an underlying distribution to assign a treatment by computing some underlying summary statistic, or \emph{property}, of the distribution over outcomes.
Property elicitation studies the relationship between the choice of objective function, treatment assignments, and various statistics.
For example, minimizing squared loss corresponds to predicting the \emph{expected value} of the outcome (the probability of repayment) and deciding whether or not to give a loan based on the expected value being above a given threshold.
This contrasts with minimizing the 0-1 loss, which corresponds to learning the \emph{mode}, of whether the person is more likely than not to repay a loan, and the assigned treatment is simply the decision to grant a loan.

In most practical optimization and allocation tasks, however, one faces constraints on the treatment space, especially when the treatments impact human livelihood and when resources are scarce.
In particular, fairness constraints are often employed to enforce the (approximately) equal algorithmic treatment of different predefined groups.
Instead of minimizing the original loss function, these algorithms often instead minimize \texttt{loss~+~weight~*~regularizer}, where the regularization term adds a penalty for violating certain desiderata about community-level outcomes.

However, to date, there is little understanding of how adding regularization functions into the optimization problem changes the property of the data distribution learned.
We give a necessary and sufficient condition for regularizers to preserve an elicited property: the property elicited by the fairness regularizer must be equivalent to the property elicited by the original loss.
However, this condition is rather strong: equivalence holds \emph{regardless of the underlying data distribution}.
Therefore, we further characterize for which data distributions the optimal treatments differ or are the same.
We demonstrate our results on group fairness regularizers, though other regularization functions can be used as well (e.g.,~\citep{mireshghallah2021privacy}).

To this end, we introduce the notion of \emph{regularized property elicitation}, and what it means for two properties to be equivalent.
In Theorem~\ref{thm:equivalence-iff} we show that, under mild conditions on the regularizer, a regularized property is equivalent to the original property if and only if the property elicited by the regularizer is equivalent to the original property.
We apply Theorem~\ref{thm:equivalence-iff} to a handful of popular fairness regularizers-- the absolute difference of demographic parity, expected equality of opportunity, and equalized false positive rates-- and demonstrate they are not equivalent to cost-sensitive classifications.
However, it is not necessarily the case that a regularizer changes the elicited property: many additive regularizers yield regularized properties equivalent to the original, namely (multi)calibration and bounded group loss.\footnote{We are not assigning a value judgment to whether or not a regularizer changes a property.}
In these cases, while the property does not change, using the regularizer is still effective because of practical limitations on the experessivity of the hypothesis class $\H$, among other optimization challenges.
It does suggest in some sense that these equivalent regularizers value ``accuracy as fairness,'' in line with sentiment from the original works.

In \S~\ref{sec:regularized-properties}, we present Theorem~\ref{thm:equivalence-iff}, which gives the necessary and sufficient condition for the equivalence of properties, and in \S~\ref{sec:nonequivalence} demonstrate these conditions on common fairness regularizers for binary classification.
For those regularizers that do change an elicited property, we additionally provide examples and geometric intuition about \emph{for which data distributions} the regularizers change (or do not change) the optimal decision, enforcing the imposed constraints.
Finally, in \S~\ref{sec:experiments}, we demonstrate our results with empirical evaluation on synthetic data, a heart attack risk analysis dataset~\citep{heart}, and the German lending dataset~\citep{kamiran2009classifying}.

\subsection{Literature review}

In machine learning, a variety of pre-, in-, and post-processing techniques have emerged in recent years to make algorithmic decision-making more fair or equitable.
We focus on one algorithmic aspect of in-processing wherein one modifies the learning algorithm itself by adding a soft constraint to the objective function, which is some weighted metric of the fairness violation.
The addition of fairness regularizers is one common approach to try to improve algorithmic decision-making in practice, though their effects are generally not well-understood (cf.~\citep{kamishima2012fairness,williamson2019fairness,denis2021fairness,bechavod2017penalizing,goel2018non,berk2017convex,do2022fair,huang2019stable,jung2020fair}).
While many proposed fairness metrics are situated in binary classification settings, extensions beyond the binary setting have been studied more recently~\citep{williamson2019fairness,denis2021fairness,konstantinov2021fairness,donini2018empirical,zafar2017fairness}.
Our framework is general enough to handle a variety of prediction tasks and regularizers beyond the fair machine learning literature.

We study the impact of regularization functions on the ``right'' decision an algorithm should make as a function of the underlying data distribution through the lens of property elicitation.
Property elicitation is well understood on an individual basis for a variety of discrete prediction tasks~\citep{lambert2008eliciting,lambert2009eliciting,finocchiaro2022embedding,lambert2018elicitation} and continuous estimation problems~\citep{steinwart2014elicitation,savage1971elicitation,brier1950verification,frongillo2014general,fissler2017higher} on an individual level.
Recently, \citet{jung2021moment} and \citet{noarov2023statistical} relate property elicitation to the notion of multicalibration.
These works extend the canonical understanding of multicalibration to estimate values beyond the mean, and provide (multicalibrated) algorithms to estimate higher moments, showing a strong equivalence between calibration and property elicitation.
These results align with some of the intuition provided in Theorem~\ref{thm:equivalence-iff}, but our result goes beyond the scope of calibration as a fairness concept.
Regularizers considering community-level outcomes and group membership requires we extend traditional notions of property elicitation.


\section{Background}\label{sec:background}
We are primarily concerned with evaluating the optimal treatment for various prediction tasks.
Consider an agent $i \in \{1, 2, \ldots, m\} = [m]$ who will achieve some outcome $y^\i \in \Y$ with probability $p^\i \in \simplex$, where $\simplex$ is the simplex over a finite set of outcomes $\Y$.
A central decision-maker (often a principal or algorithm) assigns a treatment $t^\i \in \T$ to the agent, and their error is scored according to a loss function $L: \T \times \Y \to \reals_+$.
As shorthand, denote $L(t^\i; p^\i) := \E_{Y \sim p^\i} L(t^\i, Y)$ as the expected loss over $p^\i$.
Moreover, we assume each agent $i$ is a member of a group $s^\i \in \Sc$, and want to ensure agents of different groups are treated fairly by the centralized decision-maker. 
Let $n_g := |\{ i \in [m] : s^\i = g\}|$ be the number of agents belonging to group $g$, which we assume is positive for each $g \in \Sc$.
Often, we are concerned with possibly set-valued functions, $\Gamma : \simplex \to 2^{\T} \setminus \{\emptyset\}$; for shorthand, we denote this $\Gamma : \simplex \toto \T$.

In supervised machine learning, predictions are made by learning a hypothesis function $h:\X \to \T$ mapping features $x \in \X$ to treatments $t \in \T$.
We assume $\T$ is a finite set unless otherwise stated.
If the class of hypotheses $\H$ is sufficiently expressive, then $t$ encapsulates how the optimal hypothesis \emph{should assign treatment}, given an input $x$.
Equivalently, we are concerned with optimal decisions under $p^\i = \Pr[Y \mid X = x^\i]$.
For simplicity, we abstract away $\X$ and proceed with $p^\i \in \simplex$ and $t^\i \in \T$ in the sequel.

\subsection{Regularization functions}
Often, ``fair'' algorithms constrain optimization to ensure certain desiderata are satisfied.
However, some standard optimization algorithms such as stochastic gradient descent often soften these constraints, adding an additional penalty to the loss function for violating the constraints.
We study how the addition of regularization functions $\R : \T^m \times \Sc^m \times \simplex^m \to \reals_+$  (henceforth: regularizers) change the optimal treatment assigned by minimizing the expected loss.

For example, imposing group fairness constraints, one might aim to ensure treatments are independent of the sensitive statistic (as in demographic parity) or treatments are calibrated to line up with the true probabilities of positive classification (as in multicalibration).
In this setting, given a collection of individuals $\{(s^\i, p^\i)\}$, we aim to optimize
\begin{align}\label{eq:reg-loss}
    \min_{\vec t \in \T^m} \Lreg(\vec t; \vec s; \vec p) &:= (1-\lambda) \underbrace{\left[ \frac 1 m \sum_{i=1}^m L(t^\i; p^\i) \right]}_{\text{expected loss over $m$ agents}} + \lambda \R(\vec t; \vec s; \vec p)~.
\end{align}

Because the regularizer might not be additive in $\vec t$, the treatment of an individual is not necessarily independent of the treatment of others.
This necessitates the optimization of $\vec t \in \T^m$ rather than considering each data point individually, as is standard in unregularized property elicitation.


\subsection{Property elicitation}
When making predictions, a decision-maker often aims to learn a \emph{property} $\Gamma : \simplex \toto \T$, which is simply a function mapping probability distributions to treatments.
Examples of commonly sought properties include the expected value $EV(p) = \{\E_{Y \sim p} [Y]\}$,  the mode $\mathrm{mode}(p) = \argmax_y p_y$, $\alpha$-quantiles, and rankings.

\begin{definition}[Property, elicits]\label{def:property-individual}
A \emph{property} is a function $\Gamma : \simplex \toto \T$ mapping probability distributions to reports.
If $|\T|$ is finite, we call $\Gamma$ a \emph{finite property}.
Moreover, a minimizable\footnote{One that attains the infimum in its first argument for all $y \in \Y$} loss $L : \T \times \Y \to \reals_+$ \emph{elicits} a property $\Gamma$ if, for all $p \in \simplex$,
\begin{align*}
    \Gamma(p) &= \argmin_{t \in \T} L(t; p)~.
\end{align*}
Conversely, we denote the \emph{level set} of a property $\Gamma_t = \{p \in \simplex \mid t \in \Gamma(p)\}$ as the set of distributions yielding the same optimal treatment.
\end{definition}
Throughout, we assume that properties are \emph{nonredundant}, meaning that the level set $\Gamma_t$ is full-dimensional\footnote{The affine dimension of the set equals the affine dimension of the simplex} for all $\vec t \in \T$ and for each $p \in \relint(\Gamext_{\vec t})$, we have $|\Gamext(\vec p)| = 1$.
This precludes the consideration of treatments that are rarely optimal, or only optimal if and only if another treatment is optimal as well.

Every minimizable loss elicits some property; we denote $\prop{L}$ as the (unique) property elicited by the loss $L$.
For example, the squared loss elicits the expected value~\citep{brier1950verification,savage1971elicitation}, and the level set $\Gamma_0 = \{p \in \simplex : \E_p [Y] = \{0\}\}$ of the expected value is the set of distributions with zero mean.
We will later study the geometry of the level sets of various properties to characterize the how the minimizers of unregularized losses differ from those of their regularized counterparts.
In order to do so, we consider the property $\Gamma$ evaluated on a population.
Given $\vec p \in \simplex^m$, we consider the extension $\Gamext(\vec p) := [\Gamma(p^\i)]_i$ with level sets $\Gamext_{\vec t} := \bigcap_i \{\vec p \in \simplex^m \mid t^\i \in \Gamma(p^\i)\}$.

We now extend Definition~\ref{def:property-individual} to include population-level reports for loss functions to encapsulate the case where the regularizer is not additive in $\vec t$ and/or is dependent on $\vec s$.

\begin{definition}[Regularized property elicitation]\label{def:prop-fair}
A \emph{regularized property} is a function $\Theta^{\R,\lambda} : \Sc^m \times \simplex^m \toto \T^m$ mapping beliefs over outcomes to population-level treatments.
Similarly, an objective function $L$ regularized by $\R$ (weighted by $\lambda$), denoted $\Lreg$, elicits a regularized property if, for all $\vec s \in \Sc^m$ and $\vec p \in \simplex^m$,
\begin{align*}
    \Theta^{\R,\lambda}(\vec s; \vec p) &= \argmin_{\vec t \in \T^m} \Lreg(\vec t; \vec s; \vec p).
\end{align*}
We let $\prop{\Lreg}$ denote the regularized property elicited by $\Lreg$.
\end{definition}
Denoting the level set of a regularized property requires some nuance because we are concerned with the change in optimal treatments as a function outcome distributions $\vec p$, but the regularized property is a function of $\vec s$ as well as $\vec p$.
Therefore, we denote the level set $\Theta^{\R,\lambda}_{\vec t; \vec s} = \{ \vec p \in \simplex^m \mid \vec t \in \Theta^{\R,\lambda}(\vec s; \vec p)\}$ denote the level set of the regularized property $\Theta^{\R,\lambda}$.
If $\vec s$ is clear from context, we sometimes omit it and write $\Theta^{\R,\lambda}_{\vec t}$.
We now define a trivial, constant regularizer $\R$ as non-enforcing, since it never enforces any constraints.

\begin{definition}
    A regularizer $\R$ is \emph{nonenforcing} if  $\prop{\R}_{\vec t} = \simplex^m$ for all $\vec t \in \T$, and \emph{enforcing} otherwise.
\end{definition}

\section{Equivalence of (regularized) properties}\label{sec:regularized-properties}

With an understanding of regularized property elicitation, we are now equipped to ask when a property ``changes'' with the addition of a regularizer to a loss; this requires us to consider what it means for properties to be unchanged, or equivalent.
\begin{definition}[Equivalence of properties]\label{def:property-equivalence}
A property $\Gamma : \simplex \toto \T$ \emph{is equivalent to} a regularized property $\Theta : \Sc^m \times \simplex^m \toto \T^m$ on $\vec s$ (denoted $\Gamma \equiv_{\vec s} \Theta$ or $\Gamext \equiv_{\vec s} \Theta$) if, for all $\vec p \in \simplex^m$, we have $\vec t \in \Gamma(\vec p) \iff \vec t \in \Theta(\vec s; \vec p)$.
\end{definition}

In general, but particularly for large sets of agents, equivalence of a regularized property to its unregularized counterpart is a rather strong condition: when there is a ``universally fair'' report, equivalence holds if (and only if) the regularizer elicits essentially the same property as the original loss.

The proof relies on the relationship between subgradients of the Bayes risk and property values~\citep[Theorem 4.5]{frongillo2019general}: if $L$ elicits $\Gamma$, then there is a choice of subgradients $D$ of the Bayes risk of $L$, $\risk L(\vec p) := \inf_{\vec t \in \T^m} L(\vec t; \vec p)$ such that there is a bijection from property values to $D$.
Therefore, points of nondifferentiability of $\risk L$ form the intersection of level sets.

\begin{theorem}\label{thm:equivalence-iff}
    Fix $\lambda \in (0,1)$ and $\vec s \in \Sc^m$.
    Let loss $L$ elicit $\Gamma$, $L^{\R,\lambda}$ elicit $\Theta$, and $\R$ elicit $H$.
    Then (1) $\Gamext \equivs H \implies \Gamext \equivs \Theta$.
    (2) If $H$ is nonredundant, then additionally assume $H_t \cap \Gamext_t \cap \Theta_t \neq \emptyset$ for all $t \in \T$.
    If $\Gamext \equivs \Theta$, then $\R$ is nonenforcing or $\Gamext \equivs H$.
\end{theorem}
\begin{proof}
(1)
The first statement is immediate as $H \equivs \Gamext$ implies 
    \begin{align*}
        \vec t \in \argmin_{\vec t'} \R(\vec t';\vec s; \vec p) &\iff \vec t \in \argmin_{\vec t'} L(\vec t'; \vec p) \\
        \iff \vec t \in \argmin_{\vec t'} \lambda \R(\vec t';\vec s; \vec p) &\iff \vec t \in \argmin_{\vec t'} (1 - \lambda) L(\vec t';\vec p)\\
        \implies \vec t &\in \argmin_{\vec t'} \lambda \R(\vec t'; \vec s; \vec p) + (1-\lambda) L(\vec t'; \vec p)
    \end{align*}
    Now $\vec t \in \Gamext(\vec p) \implies \vec t \in \Theta(\vec p)$.
    If $\vec t \in \Theta(\vec p)$, then consider two cases: if $\vec t \in H(\vec p)$, we are done by assumption.
    If $\vec t \not \in H(\vec p)$, then $\vec t \not \in \Gamext(\vec p)$. However, the two are equivalent, so there is some $\vec t' \in H(\vec p) \cap \Gamext(\vec p)$, so we contradict $\vec t \in \Theta(\vec p)$.

(2)
Observe that since $\T$ is finite, so is $\T^m$, and the function $\risk L : \vec p \mapsto \inf_{\vec t \in \T^m} L(\vec t; \vec p)$ is piecewise linear and concave, as it is the pointwise infimum of a finite set of affine functions (since expectation is linear).
Moreover, the function $\eta_t^L : \vec p \mapsto L(\vec t; \vec p)$ is affine and supports $\risk L$ on $\Gamext_t$ for every $\vec t \in \T^m$.
Observe that if $\Theta \equivs \Gamext$, then $\eta_t^{L}$ and $\eta_t^{\Lreg}$ support $\risk L$ and $\risk{\Lreg}$ respectively on the same sets for all $\vec t \in \T^m$.

Consider $\ND(f: \simplex^m \to \reals_+) := \{\vec p \in \simplex^m \mid f \text{ is not differentiable at } \vec p\}$.
Since $\lambda \in (0,1)$, then $\ND({\risk{L^\R}}) = \ND(\risk L) \cup \ND(\risk \R)$
\footnote{This is true regardless of $\lambda \in (0,1)$; see \cite[Lemma 5]{finocchiaro2022embedding}}.
The assumption $\Gamma \equivs \Theta$ implies that $\ND(\risk L) = \ND(\risk {L^\R})$, which in turn implies $\ND(\risk \R) \subseteq \ND(\risk L)$.
If $\R$ is enforcing and $\ND(\risk \R) \subsetneq \ND(\risk L)$, there must be some $\vec t' \in \T^m$ such that $H_{\vec t'} = \emptyset$, and therefore, $H$ is redundant.

Consider $\vec p' \in \ND(\risk L) \setminus \ND(\risk \R)$.
Observe that $\vec p' \in \Gamext_{\vec t} \cap \Gamext_{\vec t'}$ for some $\vec t \neq \vec t'$.
There exists a $\vec p \in B(\vec p', \epsilon)$ for small $\|\epsilon\| > 0$ such that $\eta^{\Lreg}_{t}$ supports $\risk{\Lreg}$ on $\mathbf{conv}(\{p, p'\})$.  
\begin{align*}
    &\vec p \in \Gamext_{\vec t'} \setminus \Gamext_{\vec t} \iff \vec p \in \Theta_{\vec t'} \setminus \Theta_{\vec t} \\
    &\implies (1-\lambda) L(\vec t'; \vec p) + \lambda \R(\vec t'; \vec p) < (1-\lambda) L(\vec t;\vec p) + \lambda \R(\vec t; \vec p) \\
    &\iff (1-\lambda) (L(\vec t'; \vec p') + c^T \epsilon) + \lambda \R(\vec t'; \vec p) < (1-\lambda) (L(\vec t;\vec p') + d^T \epsilon) + \lambda \R(\vec t; \vec p) \qquad \text{$\risk L$ affine on $\mathbf{conv}(\{p,p'\})$} \\
    &\implies \lambda \R(\vec t' ; \vec p) \leq \lambda \R(\vec t; \vec p)\qquad \text{$\epsilon \to \vec 0$}~,
\end{align*}
which implies $\vec p \in H_{t'}$, and therefore, $H_{t'} \neq \emptyset$, yielding a contradiction. 

Therefore, we must either have $\R$ nonenforcing or $\ND(\risk \R) = \ND(\risk L)$, the latter of which implies that $H$ is nonredundant.
We avoid permutations of level sets by the assumption that $H_t \cap \Gamma_t \cap \Theta_t$ is nonempty, and must have equivalence of the properties.
\end{proof}

Intuitively, Theorem~\ref{thm:equivalence-iff} says that the property elicited by a regularized loss function is the same as the unregularized loss if and only if the regularizer elicits the same property as the loss itself.
Since loss functions are measurements of accuracy, then equivalence of properties implies an algorithm values accuracy as fairness.

\section{(Non)equivalence of common fairness metrics for binary classification}\label{sec:nonequivalence}

We now evaluate a handful of common fairness regularizers, and apply Theorem~\ref{thm:equivalence-iff} to show nonequivalence between binary classification tasks and their regularized counterparts.
For each regularizer, we give restrictions on $\simplex^m$ such that the regularized property is equivalent to the original under these restrictions.

To build intuition, we examine simple cases of how regularizers change elicited properties with populations of $m=2$ agents belonging to different groups $\vec s = (a,b)$.

Figure~\ref{fig:dp_example} provides some additional intuition for the proof of Theorem~\ref{thm:equivalence-iff}.
Each subfigure gives the level sets of the property elicited by the mode regularized by the demographic parity violation (\ref{eq:01-rep-dp}), where each point in $[0,1]^2$ represents $\vec p \in \simplex^2$ by $(\Pr_{p^{(1)}}[Y=1], \Pr_{p^{(2)}}[Y=1])$.
Each colored cell depicts a different level set of a regularized property $\Theta^{DP,\lambda}$.
This regularized property is overlaid on the (unregularized) mode, so that, upon visual inspection, one observes the regions where the two properties differ.
As $\lambda \to 0$, the regularized property becomes increasingly similar to the unregularized, and as $\lambda \to 1$, the regularized property increasingly resembles the property elicited by $\R$.

\subsection{Demographic parity}\label{subsec:dp}
In the context of binary classification, one might be interested in regularizing their loss with the demographic parity violation, measured by the absolute difference of the rates at which agents are assigned the positive treatment from each of two groups.
Any treatment that assigns the positive treatment at the same rate optimizes the demographic parity regularizer, which is not equivalent to the mode.
That is, $H(\vec s; \vec p) \supseteq \{\vec 0, \ones\}$ for all $\vec p \in \simplex^m$ and $\vec s \in \Sc^m$.
Thus, if $\Sc = \{a,b\}$\footnote{This is simply for ease of exposition, and can be relaxed.}, we can apply Theorem~\ref{thm:equivalence-iff} to conclude the DP-regularized mode is not equivalent to the unregularized mode.

\begin{align*}
    L^{DP,\lambda}(\vec t; \vec s; \vec p) &= \frac {1-\lambda} m  \sum_{i = 1}^m L(t^\i; p^\i) + \lambda \left|\frac{1}{n_a} \sum_{i : s^\i = a} t^\i - \frac{1}{n_b} \sum_{i : s^\i = b} t^\i\right| \tag{DP} \label{eq:01-rep-dp}
\end{align*}

Now, with $\T = \{0,1\}$, if $L$ is the 0-1 loss\footnote{These derivations also hold if $L$ is squared loss, hinge loss, and many other losses for binary classification.}, 
we can evaluate $L^{DP,\lambda}$ for each treatment in $\T^2 = \{(1,1), (0,1), (1,0), (0,0)\}$.

\begin{align*}
    L^{DP,\lambda}((1,1); (p^{(1)}, p^{(2)})) &= \frac {1-\lambda} 2 \left[(1 - p^{(1)}) + (1 - p^{(2)}) \right] \\
    L^{DP,\lambda}((0,1); (p^{(1)}, p^{(2)})) &= \frac {1-\lambda} 2 \left[p^{(1)} + (1 - p^{(2)}) \right] + \lambda \\
    L^{DP,\lambda}((1,0); (p^{(1)}, p^{(2)})) &= \frac {1-\lambda} 2 \left[(1 - p^{(1)}) + p^{(2)} \right] + \lambda \\
    L^{DP,\lambda}((0,0); (p^{(1)}, p^{(2)})) &= \frac {1-\lambda} 2 \left[p^{(1)} + p^{(2)}\right].
\end{align*}
 These expected losses now enable us to study the level sets $\Theta^{DP,\lambda}_{\vec t; \vec s} = \{ \vec p \in \simplex ^m \mid  \vec t \in \Theta^{DP, \lambda}(\vec s; \vec p)\}$.

We have have $(0,0) \in \argmin_{\vec t \in \T^2}$ if 
\begin{align*}
    \frac {1-\lambda} 2 \left[(1 - p^{(1)}) + (1 - p^{(2)})\right] &\leq \frac {1-\lambda} 2 \left[p^{(1)} + (1 - p^{(2)})\right] + \lambda 
    \iff \frac{1 - 3\lambda} {2(1-\lambda)} \leq p^{(1)}\\
    \frac {1-\lambda} 2 \left[(1 - p^{(1)}) + (1 - p^{(2)})\right] &\leq \frac {1-\lambda} 2 \left[p^{(2)} + (1 - p^{(1)})\right] + \lambda 
    \iff \frac{1 - 3\lambda}{2(1-\lambda)} \leq p^{(2)}\\
    \frac {1-\lambda} 2 \left[(1 - p^{(1)}) + (1 - p^{(2)})\right] &\leq \frac {1-\lambda} 2 \left[p^{(1)} + p^{(2)}\right] 
    \iff p^{(1)} + p^{(2)} \leq 1~.
\end{align*}
Therefore, the level set $\Theta^{DP,\lambda}_{(0,0)}$ can be described by the polyhedron \begin{align*}\Theta^{DP,\lambda}_{(0,0)} &= \left\{p \in [0,1]^2 \mid \begin{bmatrix}0 & -1 & \frac {1-3\lambda}{2(1-\lambda)}\\-1 & 0 & \frac {1-3\lambda}{2(1-\lambda)}\\1& -1& 1 \end{bmatrix} \begin{bmatrix} p^{(1)}\\p^{(2)}\\ 1\end{bmatrix} \geq \vec 0\right\}~.
\end{align*}
Observe that the final constraint is actually one on the marginal $P[Y]$: the expected outcome over the whole population should be less likely to be $1$ than $0$.
We can evaluate the rest of the level sets in a similar manner.


Now let us gain some geometric intuition for how these level sets change by referencing Figure~\ref{fig:dp_example}.
For two agents belonging to different groups, each point in the figure represents a pair $\vec p := (p^{(1)}, p^{(2)})$ of true probabilities for the two agents.
The pair $\vec p \in [0,1]^2$, and the region $[0,1]^2$ can be divided into up to $|\T^m|$ regions for which each $\vec t \in \T^m$ is contained in $\Theta^{DP,\lambda}(\vec p)$.
The sequence of figures in Figure~\ref{fig:dp_example} denotes the level sets of $\Theta^{DP,\lambda}$ as one varies $\lambda \in [0,1]$.
For intuition, one can observe that the regions where the players receive the same treatment (blue and red) grows as $\lambda$ increases, starting with $1/2$ of the $[0,1]^2$ space, and increasing to all of $[0,1]^2$ as $\lambda \to 1$.

\begin{figure*}
\centering
\includegraphics[width=0.7\linewidth]{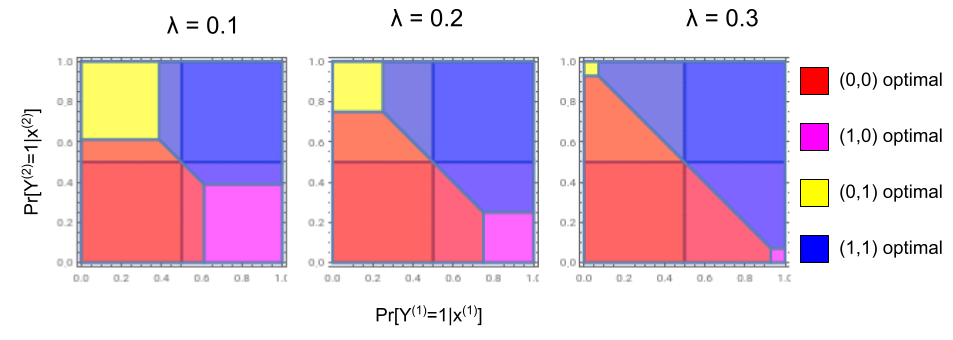}
    \caption{Visualizing the level sets of the $DP$-regularized property $\Theta^{DP,\lambda}$ for different values of $\lambda \in [0,1]$, where $m = 2$ and $\vec s = (a,b)$. Each point $(p^{(1)}, p^{(2)})$ in a square represents $(\Pr_{p^{(1)}}[Y=1], \Pr_{p^{(2)}}[Y=1])$, and each colored cell represents sets of $(p^{(1)}, p^{(2)})$ pairs such that the optimal treatment is the same for all points in the cell. For example, the magenta cell (lower right) is the set of distributions where the decision-maker prefers to attribute the positive treatment ($t^\i = 1$) to the first agent, and the negative treatment ($t^\i = 0$) to the second agent.}
    \label{fig:dp_example}
\end{figure*}

We now turn our attention towards the regions of $\simplex^m$ where the regularized and unregularized properties are equivalent with a demographic parity regularizer.
First, we observe that if uniform treatment of a population is optimal on the unregularized property, it is also optimal on the regularized property.

\begin{proposition}\label{thm:subset-intersection}
    Fix $\lambda \in (0,1)$.
    Let $L$ elicit $\Gamma$, $\Lreg$ elicit $\Theta$, and $\R$ elicit $H$.
    For all $\vec t \in \T^m$ and $\vec s \in \Sc^m$, $\Gamext_{\vec t} \cap H_{\vec t; \vec s} \subseteq \Theta_{\vec t}$.
\end{proposition}
\begin{proof}
    $\vec p \in \Gamext_{\vec t} \cap H_{\vec t; \vec s} \implies L(\vec t; \vec p) \leq L(\vec t'; \vec p)$ and $\R(\vec t; \vec s; \vec p) \leq \R(\vec t'; \vec s; \vec p)$ for all $\vec t' \in \T^m$, which in turn implies $L(\vec t; \vec p) + \R(\vec t; \vec s; \vec p) \leq L(\vec t'; \vec p) + \R(\vec t'; \vec s; \vec p) \implies (1-\lambda)L(\vec t; \vec p) + \lambda \R(\vec t; \vec s; \vec p) \leq (1-\lambda) L(\vec t'; \vec p) + \lambda \R(\vec t'; \vec s; \vec p)$ for all $t' \in \T^m$.
\end{proof}

We apply this result to the ``universally fair'' reports via demographic parity $\vec 0$ and $\ones$.

\begin{corollary}
    Fix $\vec s \in \Sc^m$ and $\lambda \in [0,1]$.
    Let $L$ elicit $\Gamma$ and $\L^{DP,\lambda}$ elicit $\Theta$.
    $\Gamext_{\vec 0}\subseteq \Theta_{\vec 0; \vec s}$.
    Moreover, $\Gamext_{\ones}\subseteq \Theta_{\vec s; \ones}$.
\end{corollary}
\begin{proof}
    Let $H := \prop{L^{DP,\lambda}}$.
    For all $\vec p \in \simplex^m$, we have $\{\vec 0, \vec 1\} \subseteq H(\vec p)$.
    Therefore, $\Gamext_{\vec 0} \cap H_{\vec 0} = \Gamext_{\vec 0}$ (and similarly with $\Gamext_\ones \cap H_\ones$).
    Therefore, $\Gamext_{\vec 0} = \Gamext_{\vec 0} \cap H_{\vec 0} \subseteq \Theta_{\vec 0}$ and $\Gamext_{\ones} = \Gamext_{\ones} \cap H_{\ones} \subseteq \Theta_{\ones}$.
\end{proof}

We now turn our attention to the opposite case: if, while regularized, treating different groups differently (and uniformly within the groups) is optimal, then it is also optimal in the unregularized setting.
In particular, this holds for treatments maximizing $\R$.
\begin{proposition}
    Fix $\vec s \in \{a,b\}^m$ and $\lambda \in [0,1]$.
    Fix $\vec t = \ones_a$ (or $\ones_b$ without loss of generality).
    Let $L$ elicit $\Gamma$ over outcomes $\Y = \{0,1\}$.
    $\Theta^{DP,\lambda}_{\vec t; \vec s} \subseteq \Gamext_{\vec t}$.
\end{proposition}
\begin{proof}
    With $\vec s$ fixed, $t \in \argmax_{\vec t'} DP(\vec t';\vec p)$ for all $\vec p \in \simplex^m$.
    Therefore,
    \begin{align*}
        (1-\lambda) L(\vec t;\vec p) + \lambda DP(\vec t; \vec p) &\leq (1-\lambda) L(\vec t'; \vec p) + \lambda DP(\vec t'; \vec p) \forall \vec t' \\
        \implies (1-\lambda)L(\vec t;\vec p) &\leq (1-\lambda) L(\vec t';\vec p) \qquad \forall \vec t'~,
    \end{align*}
    which implies the result.
\end{proof}

With that, we partially characterize the relationship between the unregularized and DP-regularized level sets for standard binary classification.
In the simple case with $m=2$ agents, this characterization is complete: if the optimal treatment is uniform, it stays uniform.
Moreover, if the most ``unfair'' treatment wherein all the members of one group receive the treatment, and none of the second group is optimal in the regularized setting, it is also optimal in the unregularized setting.
In any other setting, the optimal treatment changes with the addition of a DP regularizer.

\subsection{Equalized FPR}\label{subsec:fpr}
Following a similar process to \S~\ref{subsec:dp}, we now consider the regularizer that measures the absolute difference of false positive rates across groups, where the false positive rate is given by $FPR_g(\vec t; \vec s; \vec p) = \Pr[Y^\i = 0 \mid t^\i = 1, s^\i = g] = \frac 1 {|\{i : t^\i = 1, s^\i = g\}|} \sum_{i : s^\i = g, t^\i = 1} (1 - p^\i)$.
The optimization problem then becomes
\begin{align}
    L^{FPR,\lambda}(\vec t; \vec s; \vec p) 
    &= \frac {1-\lambda} m \sum_i L(t^\i; p^\i)  + \lambda \left| FPR_a(\vec t; \vec s; \vec p) - FPR_b(\vec t; \vec s; \vec p) \right|     \tag{FPR}\label{eq:FPR-loss}
\end{align}

The FPR regularizer computes the difference of false positive rates between groups, so one can observe that the false positive rate of a group is is reduced by assigning more negative treatments $t^\i = 0$.
We can see in Figure~\ref{fig:fpr_example} that the FPR regularizer then makes it worse for an algorithm to assign the positive treatment to an agent $i$ even if $p^\i$ is slightly greater than $1/2$, as marked by the $\star$ in Figure~\ref{fig:fpr_example}(R). 

As in \S~\ref{subsec:dp}, we can apply Proposition~\ref{thm:subset-intersection} to show that if assigning everyone the negative treatment is optimal in the unregularized setting, it is also the optimal treatment with the FPR regularizer.
\begin{corollary}
    Fix $\vec s \in \Sc^m$.
    Let $L$ elicit $\Gamma$ and $L^{FPR,\lambda}$ elicit $\Theta^{FPR,\lambda}$.
    $\Gamext_{\vec 0} \subseteq \Theta^{FPR,\lambda}_{\vec 0; \vec s}$.
\end{corollary}
\begin{proof}
    For all $\vec p \in \simplex^m$, we have $\vec 0 \in H(p)$.
    Therefore $\Gamext_0 = \Gamext_0 \cap H_{\vec 0; \vec s} \subseteq \Theta_{\vec 0}$ by Proposition~\ref{thm:subset-intersection}.
\end{proof}

\begin{figure*}
\centering
\includegraphics[width=0.7\linewidth]{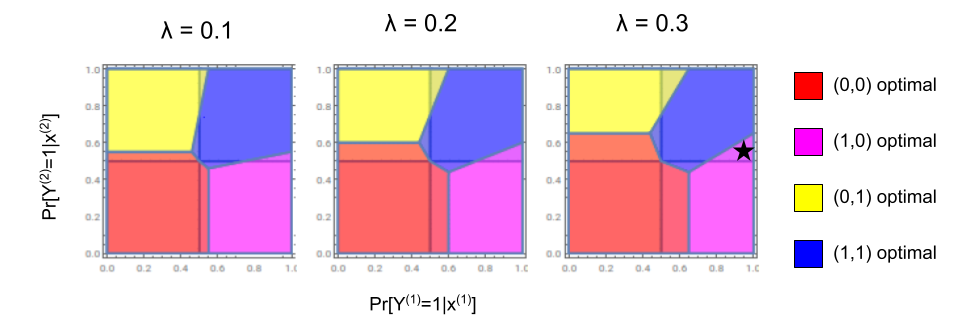}
    \caption{Visualizing the level sets of the $FPR$-regularized property $\Theta^{FPR,\lambda}$ for different values of $\lambda \in [0,1]$, where $m = 2$ and $s = (a,b)$. Each point $(p^{(1)}, p^{(2)})$ in a square represents $(\Pr_{p^{(1)}}[Y=1], \Pr_{p^{(2)}}[Y=1])$, and each colored cell represents sets of $(p^{(1)}, p^{(2)})$ pairs such that the optimal treatment is the same for all points in the cell. For example, the magenta cell (lower right) is the set of distributions where the decision-maker prefers to attribute the positive treatment ($t^{(1)} = 1$) to the first, and the negative treatment ($t^{(2)} = 0$) to the second agent.}
    \label{fig:fpr_example}
\end{figure*}

\subsection{Expected equality of opportunity}
While standard equality of opportunity (cf.~\citep{hardt2016equality}) requires access to observed labels, we are interested in equality of opportunity in expectation, and consider a variant that does not require access to labels proposed by \citet{blandin2022fairness}.
Consider the treatment space $\T = \{0,1\}^m$ and regularizer $\R(\vec t; \vec s; \vec p) = |EEO_a(\vec t; \vec s; \vec p) - EEO_b(\vec t; \vec s; \vec p)|$, where 
\begin{align}
    EEO_g(\vec t; \vec s; \vec p; g) &= \Pr_{i \sim [m]}[t^\i = 1 \mid y^\i = 1, s^\i = g] \tag{EEO}\label{eq:EEO}\\
    &= \frac{\Pr[Y^\i = 1 \mid t^\i = 1, s^\i = g]\Pr[t^\i = 1]}{\Pr[Y^\i = 1]} \nonumber\\
    &= \frac{\left(\frac 1 {|\{i : s^\i = g, t^\i = 1\}|}\sum_{i : t^\i = 1, s^\i = g}(p^\i)\right)(\sum_i t^\i)}{\sum_i p^\i}~\nonumber.
\end{align}

We can apply Proposition~\ref{thm:subset-intersection} to show that uniform treatment being optimal in the unregularized case implies it is also optimal with the EEO regularizer as well.

\begin{corollary}
    Fix $\vec s \in \Sc^m$, and let $L$ elicit $\Gamma$ over outcomes $\Y = \{0,1\}$ and $L^{EEO, \lambda}$ elicit $\Theta^{EEO, \lambda}$.
    $\Gamext_{\vec 0} \subseteq \Theta^{EEO,\lambda}_{\vec 0}$.
\end{corollary}


\begin{figure*}
\centering
\includegraphics[width=0.7\textwidth]{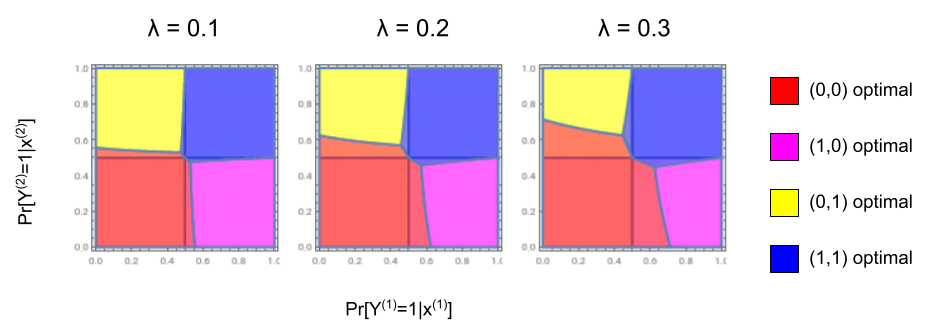}
    \caption{Level sets of the EEO-regularized mode on $\Y = \{0,1\}$.}
    \label{fig:eeo_example}
\end{figure*}




\subsection{Equivalent regularizers}\label{sec:equivalent}

In the previous section, we use Theorem~\ref{thm:equivalence-iff} to show the nonequivalence of reguliarized properties, and examine a few common regularizers to show some restrictions that recover equivalence under certain distributional assumptions on the outcomes.
We examine two regularizers that elicit the mode, and thus the regularized property is equivalent to the unregularized on all of $\simplex^m$: (multi)calibration~\citep{pleiss2017fairness,hebert2018multicalibration,jung2021moment,noarov2023statistical} and bounded group loss~\citep{agarwal2019fair}.
In some sense, this suggests that these regularizers value accuracy as fairness.
If models are as accurate as they could possibly be, the most ``fair'' treatments to assign are also the most accurate.
In practice, the regularizers mitigate unfairness arising from limited expressivity of the model: if the model was perfectly expressive and could predict the mode perfectly, it would assign the same treatments even with heavy penalties for ``unfairness.''

\paragraph{Calibration}
Calibration constraints ensure that the predicted value $t^\i$ most closely lines up with the true probability $p^\i$, regularizing the loss by the sums of the absolute differences $|t^\i - p^\i|$.
The absolute difference elicits the $1/2$-quantile, which is also the mode on $\Y = \{0,1\}$, so the regularizer $\R(\vec t;\vec s; \vec p) = \sum_g \frac 1 {n_g}\sum_{i:s^\i = g} |t^\i - p^\i|$ elicits the mode in binary classification problems.

Formally, consider the objective
\begin{align}
    L^{Cal,\lambda}(\vec t; \vec s; \vec p) &= \frac {1-\lambda} m \sum_i L(t^\i; p^\i) + \lambda \sum_g \frac 1 {n_g} \sum_{i : s^\i = g} |t^\i - p^\i| \tag{Cal}\label{eq:calibration}
\end{align}
This constraint does not include any comparisons across group averages, so the optimal report is obtained by giving individual predictions.
In binary classification, the $1/2$-quantile is the same as the mode, so the property is given $\Theta^{Cal,\lambda}(\vec s; \vec p) = \mathrm{mode}(\vec p)$.

This observation holds even with different weightings for specific subgroups, as in multicalibration a l\'a \citet{hebert2018multicalibration}.


\paragraph{Bounded group loss}
We now consider the constraint on bounded group loss: $\E_{Y \mid S = s}L(r, Y) < \epsilon$ for all $s \in \Sc$, introduced by~\citet{agarwal2019fair}.
To model bounded group loss as a soft constraint, we simply weigh the expected loss conditioned on the group size as a regularizer, so accuracy is more incentivized on small groups.
\begin{align*}
    L^{BGL,\lambda}(\vec t; \vec s; \vec p) &= \frac {1-\lambda} m \sum_i L(t^\i; p^\i ) + \sum_g \frac \lambda {n_g} \sum_{i: s^\i = g}  L(t^\i; p^\i)
\end{align*}

Adding this constraint as a fairness regularizer does not change the property elicited (e.g., $\Theta^{BGL,\lambda}(\vec s; \vec p) = \Gamext(\vec p)$ for all $p \in \simplex^m$).
In part this is because it still encourages the model to learn what is best for each individual in the population, where other constraints add a regularizer that compares the deviation between two groups.

\begin{corollary}\label{cor:bgl-equiv}
Let $L$ elicit $\Gamma$.
$\Gamext \equivs \Theta^{BGL, \lambda}$ for all $\vec s \in \Sc^m$ and $\lambda \in [0,1]$.
\end{corollary}
\begin{proof}
The regularizer $\R(\vec t; \vec s; \vec p) := \sum_g \frac 1 {n_g} \sum_{i : s^\i = g} L(t^\i; p^\i)$ is additive in $\vec t$, and elicits the same property as $L$ since it is simply a reweighing of $L$.
\end{proof}







\section{Experiments}\label{sec:experiments}
While property elicitation allows us to reason about what a treatment an algorithm \emph{should} assign, we examine whether or not these decisions are consistent with the treatments assigned by algorithms in practice with simple models.
We first generate a set of synthetic datasets to understand how a classifier's decisions change as one navigates the space of data distributions.
Moving through this space demonstrates the relationship between loss and regularizer in the synthetic setting as the data distribution over changes in $\simplex^m$.
We then evaluate the effect of the regularizer weight $\lambda$ on treatment assignment in cardiovascular disease risk prediction~\citep{heart} and lending~\citep{kamiran2009classifying} datasets, where the data distribution is fixed.
In both settings, we train a linear classifier over $30$ trials with binary cross entropy loss with (a) no regularizer, (b) demographic parity difference (c) false positive rate difference (d) false negative rate difference, (e) equality of opportunity difference, and compute the fairness violations of the classifier trained on each of the four losses, where elicited property values are shown in Figure~\ref{fig:level_sets_lambda}.

\subsection{Effect of the data distribution}\label{subsec:data-distribution-experiments}
Recall that we applied Theorem~\ref{thm:equivalence-iff} and its intuition in Figures~\ref{fig:dp_example}, \ref{fig:fpr_example}, and \ref{fig:eeo_example} to conclude the mode is not equivalent to $\Theta^{\R,\lambda}$ for various regularizers including \eqref{eq:01-rep-dp}, \eqref{eq:FPR-loss}, False Negative Rates (in \S~\ref{app:omitted-regs}), and Expected Equality of Opportunity~\eqref{eq:EEO}.
However, the equivalence of regularized properties and their unregularized counterparts is a rather strong condition, as pointwise equivalence must hold for \emph{every} set of data distributions.
In practice, the true data distribution may be somewhere in the space of distributions where the property value does not change for the chosen value of $\lambda$. 
With the knowledge that equivalent distributions have no endogeneous differences in hand, we generate a set of synthetic distributions to understand tradeoffs to regularizers as we move though the space of data distributions.

We generate generate synthetic datasets for binary classification as follows: there are two groups, $\Sc = \{a, b\}$ with $\Pr[a] = \Pr[b] = 1/2$, a member of each group has $\Pr[Y = 1 \mid S = g] = p_g \in [0,1]$.
Each set of agents is represented by $x = \{p_a, p_b, r_1, \ldots, r_k\}$, where $r_1, \ldots, r_k$ are uniformly random values in $[-1,1]$.
We then train a logistic regressor via stochastic gradient descent (30 trials with learning rate = 0.001, 1500 epochs, 10000 $(p_a, p_b)$ pairs, $k = 3$), that minimizes the binary cross entropy loss regularized by either demographic parity, false positive rate, false negative rate, or difference in equality of opportunity with $\lambda = 0.15$.
The simplicity of features is intentional: the ``perfect'' decision should be fully realizable in the unregularized setting, so the benchmark accuracy should be relatively high.
Fixing the probability for a positive outcome $p_a = 0.3$ for a member of group $a$, we vary the probability of a positive outcome $p_b$ for a member of group $b$ to observe how fairness violations change as the underlying data distribution changes.
For intuition, by design of the datasets, we reason about the ``average member'' of the population and reference the level sets drawn in Figure~\ref{fig:synthetic_compare_props}.
Fixing $p_a$ and varying $p_b$ can be thought of as understanding what happens in decision making as one moves vertically up the line $\{(0.3, p_b)\mid p_b \in [0,1]\}$, denoted by the black dashed lines in Figure~\ref{fig:synthetic_compare_props}.
In Figure~\ref{fig:violations_synthetic_simple} (L), we observe a a significant difference in DP violation rate only when $p_b \geq 1/2$, in line with the intuition in Figure~\ref{fig:synthetic-compare-props}.
Similarly, the false positive rate violation gap ``opens up'' for $p_b \in [0.5, 0.65]$ in Figure~\ref{fig:violations_synthetic_simple} (ML), in line with Figure ~\ref{fig:synthetic_compare_props}, and no significant different in FNR violations is observed as decision-making on this axis does not change in Figure~\ref{fig:synthetic_compare_props}.
Finally, for EEO, this gap opens for $p_b \geq 1/2$, then closes again later.

\begin{figure}
\begin{minipage}{0.24\linewidth}
    \centering
    \includegraphics[width=\linewidth]{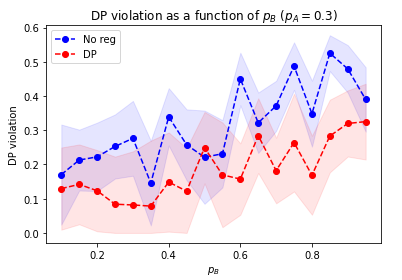}
    \label{fig:synthetic_dpviol_simple}
\end{minipage}
\hfill
\begin{minipage}{0.24\linewidth}
    \centering
    \includegraphics[width=\linewidth]{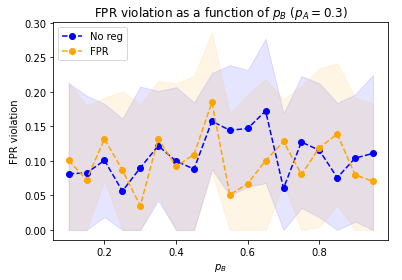}
    \label{fig:synthetic_fprviol_simple}
\end{minipage}
\begin{minipage}{0.24\linewidth}
    \centering
    \includegraphics[width=\linewidth]{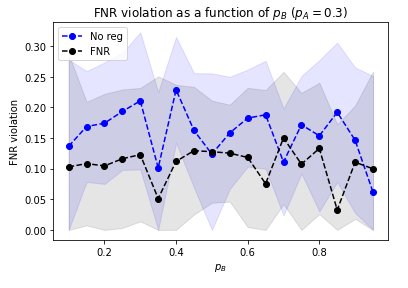}
    \label{fig:synthetic_fnrviol_simple}
\end{minipage}
\begin{minipage}{0.24\linewidth}
    \centering
    \includegraphics[width=\linewidth]{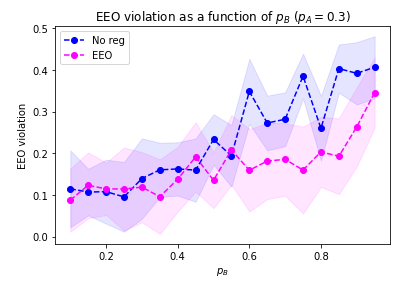}
    \label{fig:synthetic_eeoviol_simple}
\end{minipage}
\caption{Regularizer values with synthetic data generated via $\Pr[Y=1 \mid g = a] = 0.3$ and $\Pr[Y=1 \mid g = b]$ on the horizontal axis.}
\label{fig:violations_synthetic_simple}
\end{figure}

\begin{figure}
    \centering
    \includegraphics[width=\linewidth]{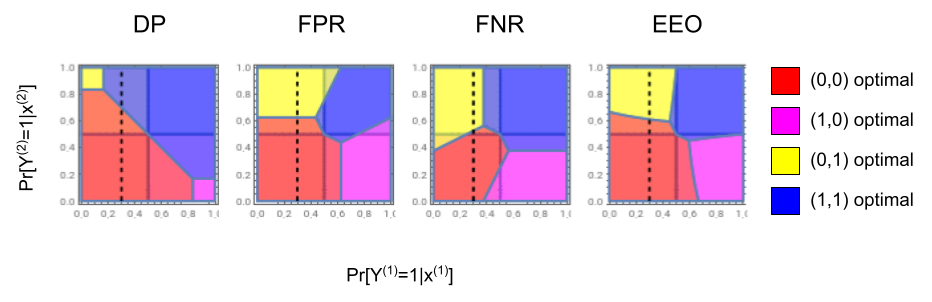}
    \caption{Fixing $p_a = 0.3$, examining how the property value changes as a function of $p_b$ for different regularizers. Demographic parity results in different decisions only if $p_b \in [1/2, 3/4]$, FPR if $p_b \in [1/2, 2/3]$, FNR has essentially the same property values on the line $p_a = 0.3$, and EEO leads to a small region where optimal decisions change for $p_b \in [1/2, 2/3]$.}
    \label{fig:synthetic_compare_props}
\end{figure}

\subsection{The effect of choice of $\lambda$}
Conversely to the interpretation of the experiments in \S~\ref{subsec:data-distribution-experiments}, to gain intuition for why decisions might change as a function of $\lambda$, we now consider each dataset representing a $(p_a, p_b)$ point in one of Figures~\ref{fig:dp_example}--\ref{fig:eeo_example}, and consider how the level set it belongs to changes as one changes $\lambda$.
We examine two datasets, German lending~\citep{kamiran2009classifying} and heart disease risk prediction~\citep{heart}. 
For both datasets, we train 30 linear models with 15000 epochs, learning rate of $0.001$.

\paragraph{German lending}
In the German lending dataset, we treat age as the sensitive attribute, using an indicator thresholded at 25 years old. On the entire dataset, we have $\Pr[Y = 1 \mid S \geq 25] = 0.728$ and $\Pr[Y = 1 \mid S < 25] = 0.578$, and an unbalanced group representation with $\Pr[S < 25 ] = 0.191$.

Perhaps surprisingly, we observe little impact of the choice of $\lambda$; moreover, in Figure~\ref{fig:violations_german}, we observe no significant difference in the performance across fairness metrics from regularized and unregularized losses.
Upon closer inspection, this can be explained partly by the observation that the ``average'' group members $(p_a, p_b) = (0.728, 0.578)$: a distribution that warrants treating the average member of each subpopulation the same, which aligns with most fairness regularizers.
This is demonstrated in Figure~\ref{fig:level_sets_lambda}, where the $(p_a, p_b)$ coordinate is denoted by a \emph{g}, for German.
For every subfigure in Figure~\ref{fig:level_sets_lambda}, the \emph{g} coordinate is in the blue cell, implying that the ``average member'' of each group receives the same treatment with a fairness-regularized loss as with an unregularized loss, suggesting that for the probability distribution underlying this dataset, data subjects are already treated approximately fairly by the unregularized loss.

\begin{figure}
\begin{minipage}{0.19\linewidth}
    \centering
    \includegraphics[width=\linewidth]{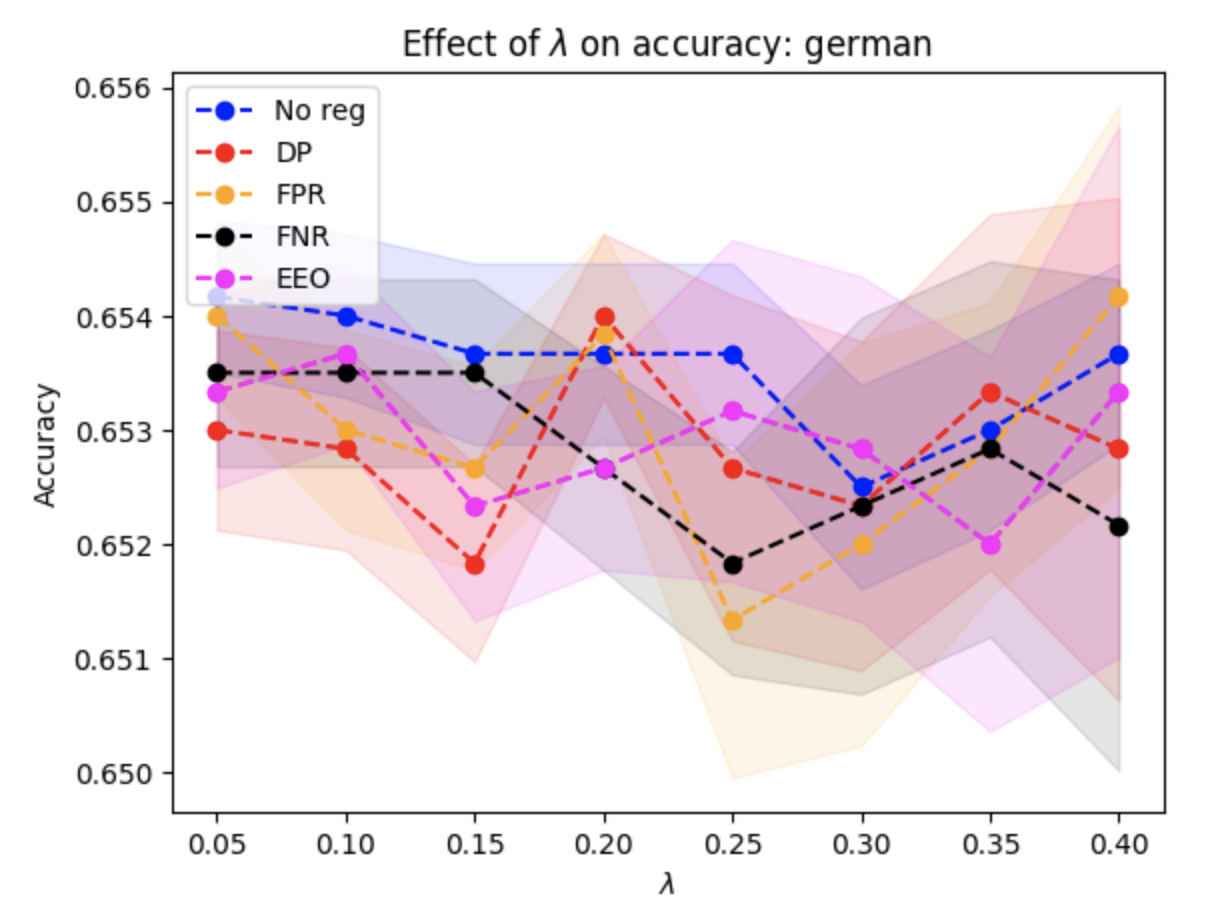}
    \label{fig:german_acc}
\end{minipage}
\hfill
\begin{minipage}{0.19\linewidth}
    \centering
    \includegraphics[width=\linewidth]{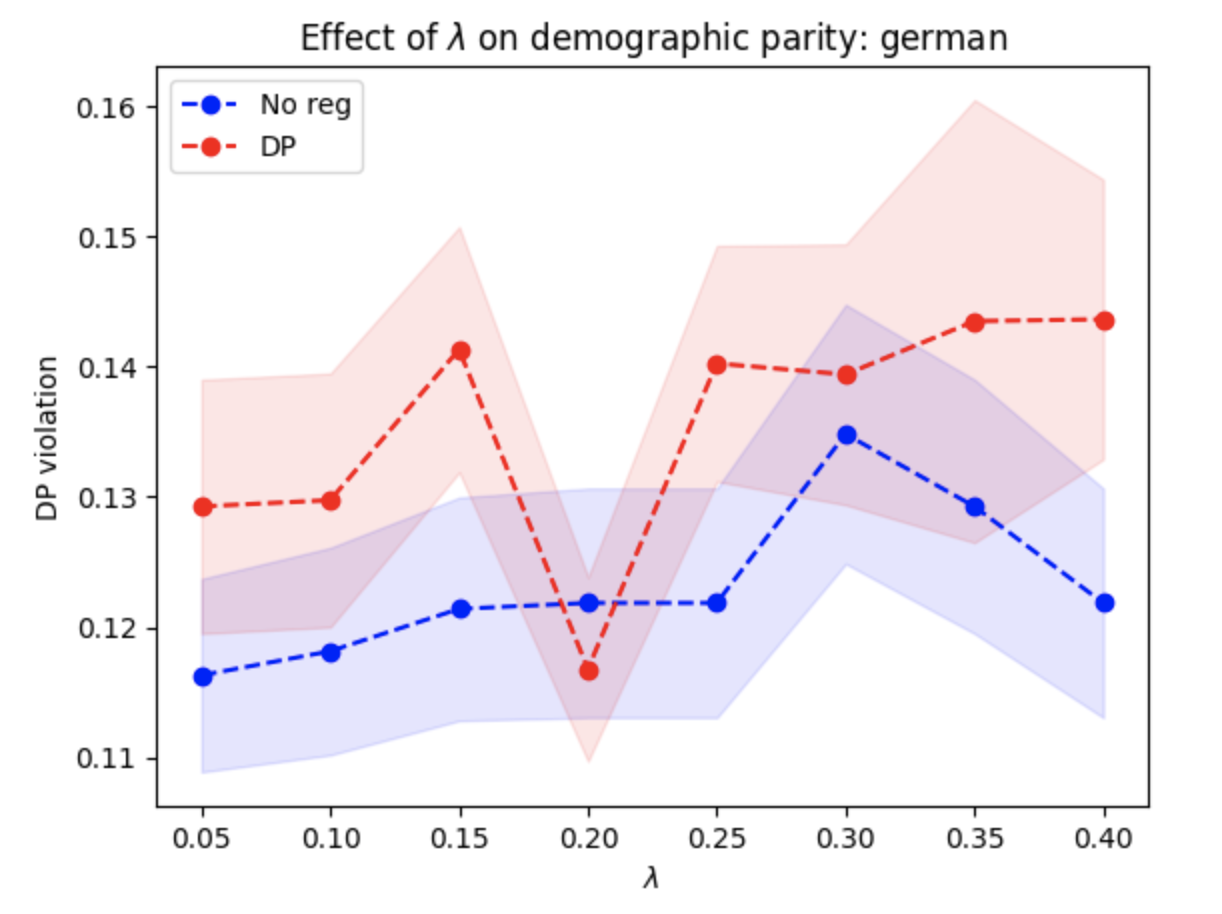}
    \label{fig:german_DP}
\end{minipage}
\hfill
\begin{minipage}{0.19\linewidth}
    \centering
    \includegraphics[width=\linewidth]{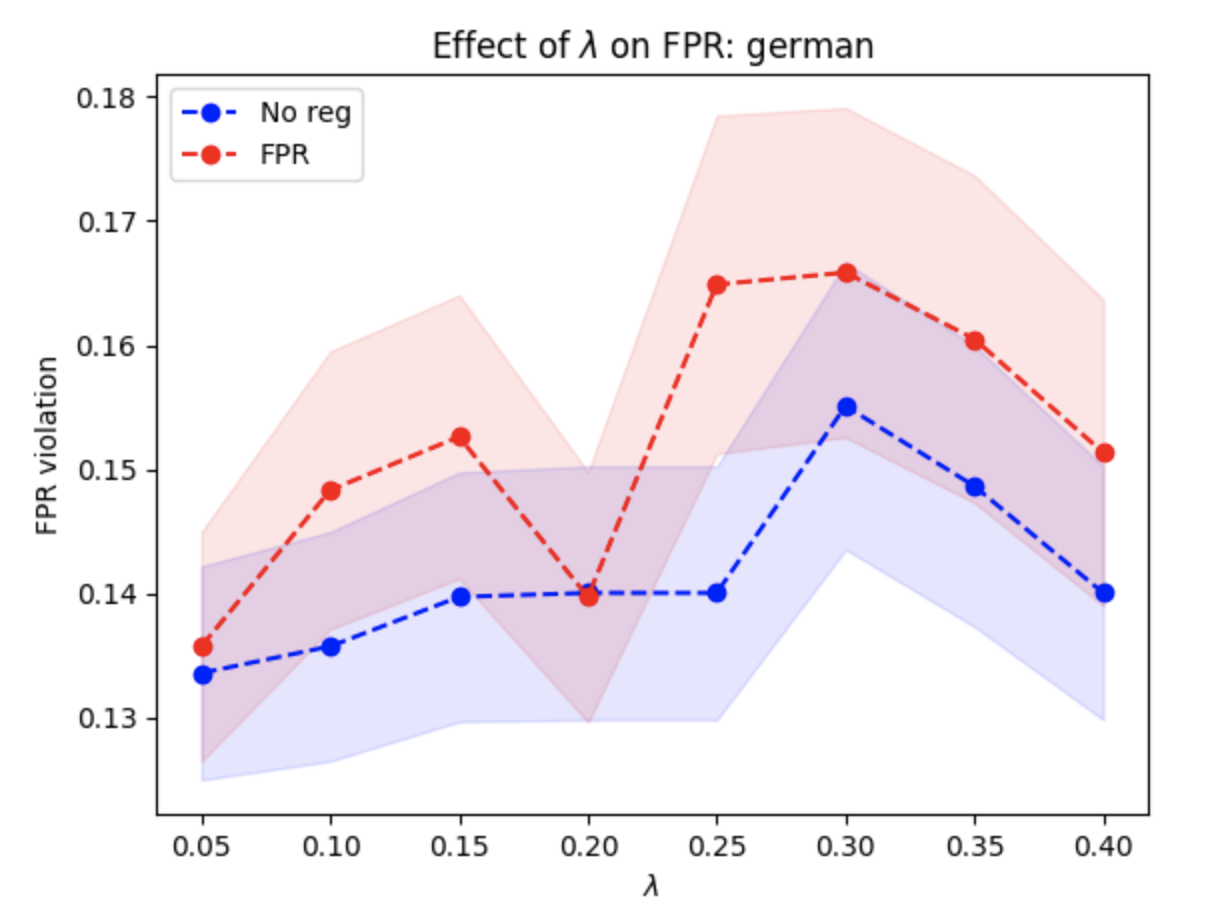}
    \label{fig:german_FPR}
\end{minipage}
\hfill
\begin{minipage}{0.19\linewidth}
    \centering
    \includegraphics[width=\linewidth]{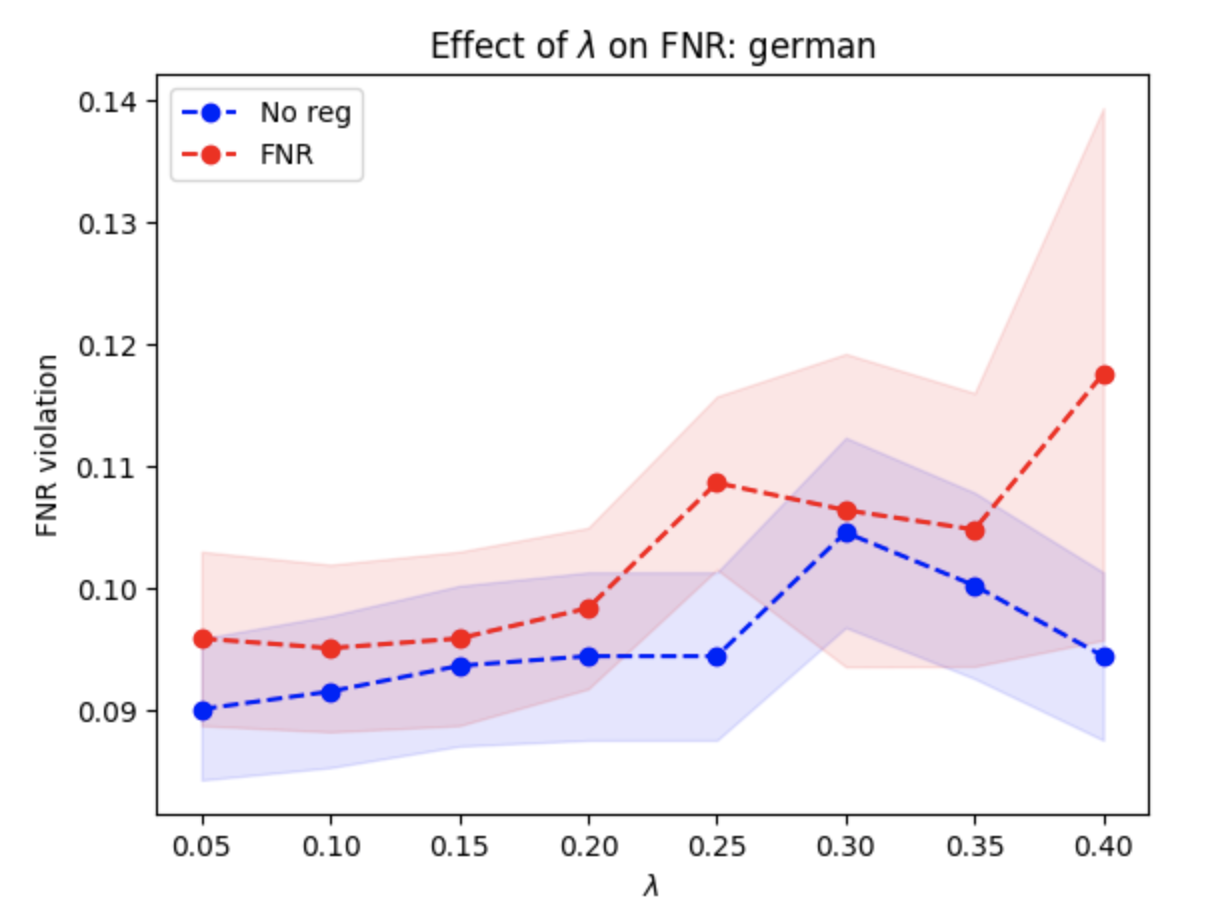}
    \label{fig:german_FNR}
\end{minipage}
\hfill
\begin{minipage}{0.19\linewidth}
    \centering
    \includegraphics[width=\linewidth]{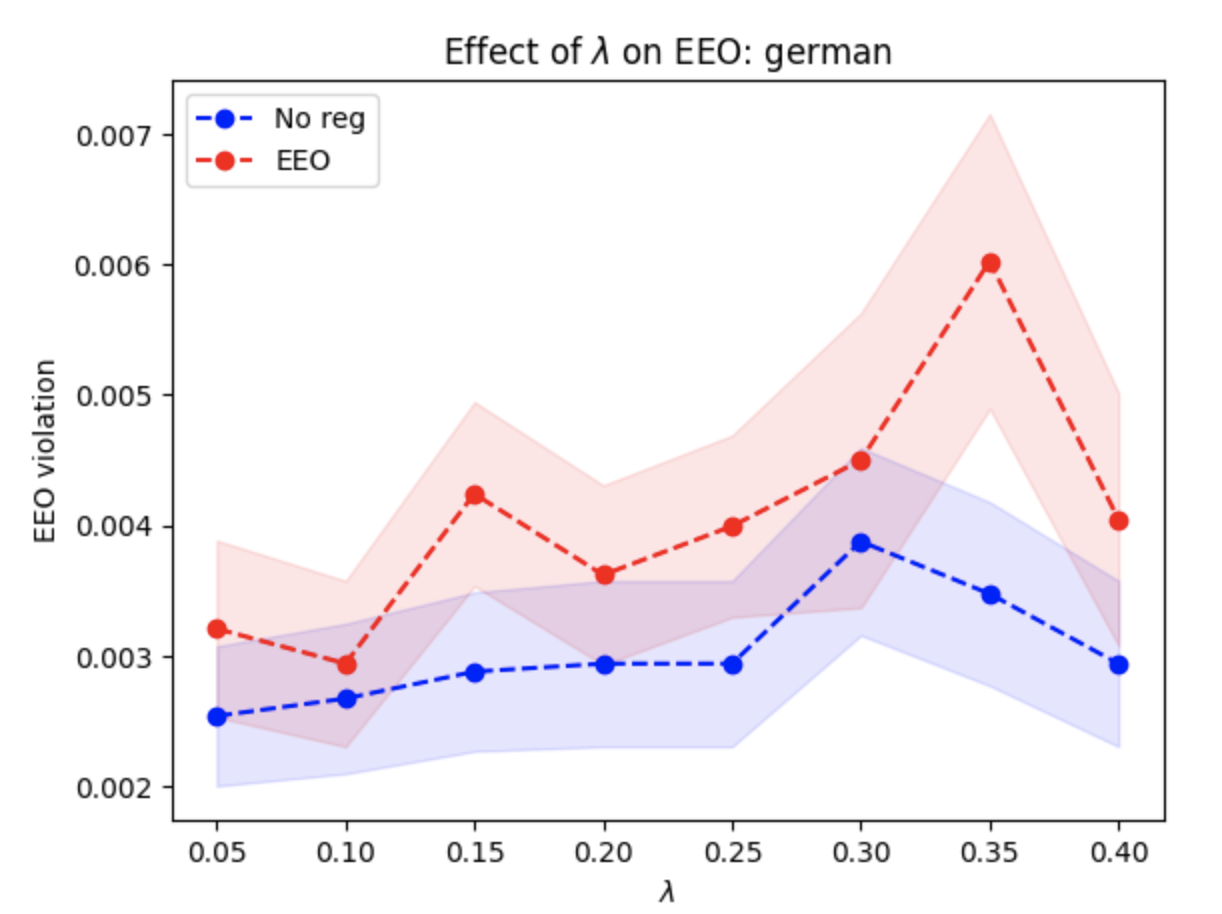}
    \label{fig:german_EEO}
\end{minipage}
\caption{Effect of $\lambda$ on regularizer values on the German lending dataset~\citep{kamiran2009classifying}. Because the $(p_a, p_b)$ point summarizing group differences in the dataset are at a point where regularized decisions are the same as unregluarized decisions, it is unsurprising that regularizers do not significantly reduce unfairness, regardless of $\lambda$.}
\label{fig:violations_german}
\end{figure}

\begin{figure}
\begin{minipage}{0.19\linewidth}
    \centering
    \includegraphics[width=\linewidth]{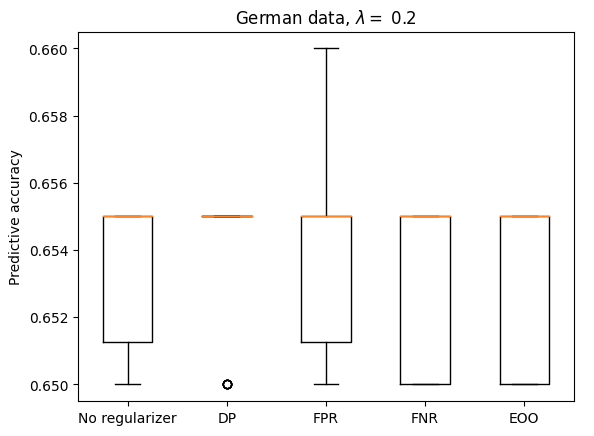}
    \label{fig:german_acc_dist}
\end{minipage}
\hfill
\begin{minipage}{0.19\linewidth}
    \centering
    \includegraphics[width=\linewidth]{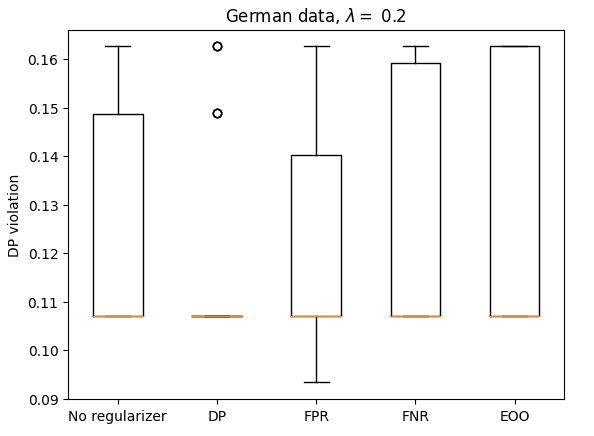}
    \label{fig:german_DP_dist}
\end{minipage}
\hfill
\begin{minipage}{0.19\linewidth}
    \centering
    \includegraphics[width=\linewidth]{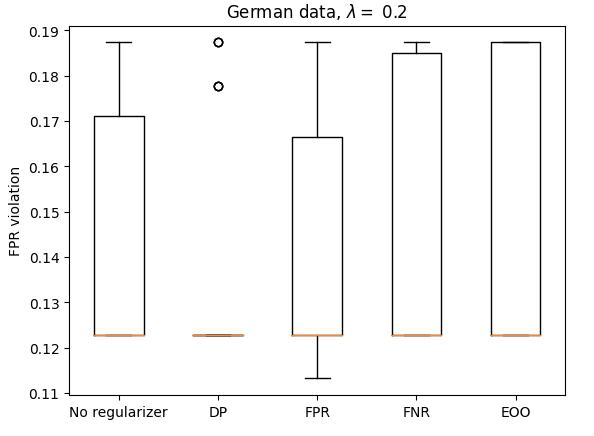}
    \label{fig:german_FPR_dist}
\end{minipage}
\hfill
\begin{minipage}{0.19\linewidth}
    \centering
    \includegraphics[width=\linewidth]{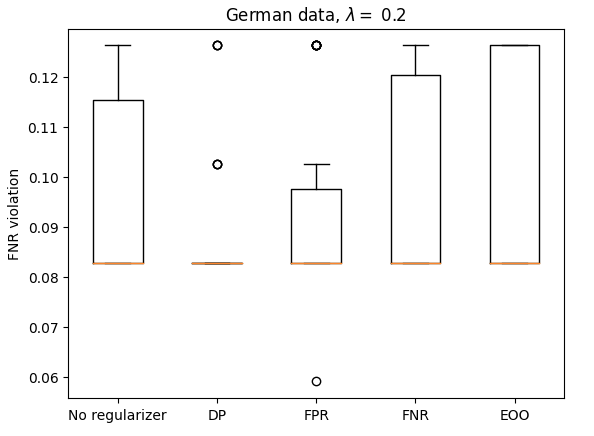}
    \label{fig:german_FNR_dist}
\end{minipage}
\hfill
\begin{minipage}{0.19\linewidth}
    \centering
    \includegraphics[width=\linewidth]{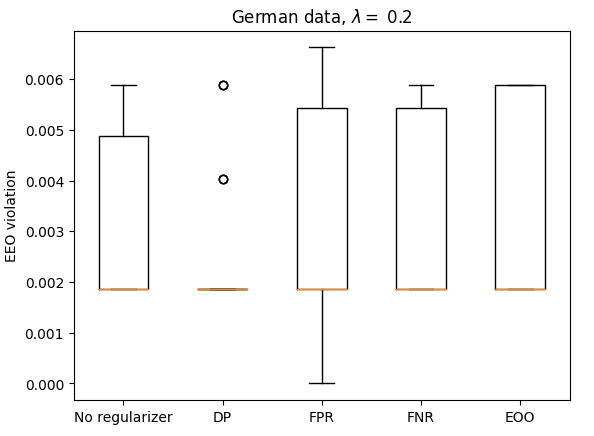}
    \label{fig:german_EEO_dist}
\end{minipage}
\caption{Distributions of accuracy and fairness violations in lending data. In general, it seems the models are tending to make similar predictions, which often nearly equal medians.}
\label{fig:violations_german_dists}
\end{figure}

\begin{figure}
\centering
\begin{minipage}{0.7\linewidth}
    \centering
    \includegraphics[width=\linewidth]{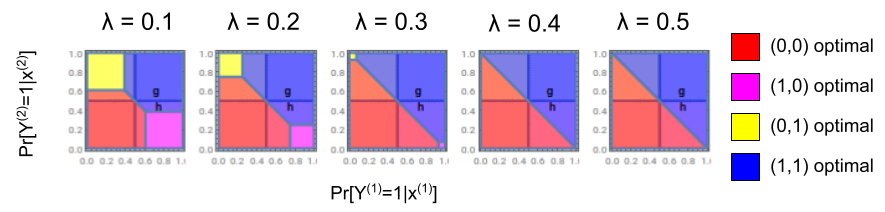}
    \label{fig:dp_ls_lambda}
\end{minipage}
\hfill
\begin{minipage}{0.7\linewidth}
    \centering
    \includegraphics[width=\linewidth]{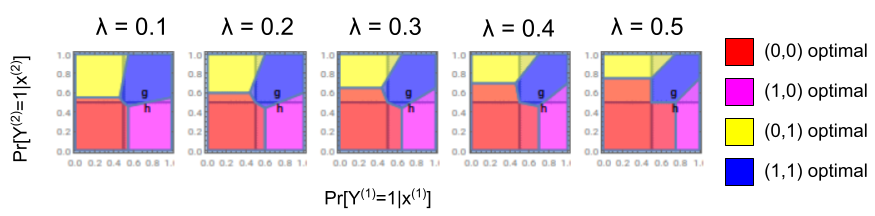}
    \label{fig:fpr_ls_lambda}
\end{minipage}
\hfill
\begin{minipage}{0.7\linewidth}
    \centering
    \includegraphics[width=\linewidth]{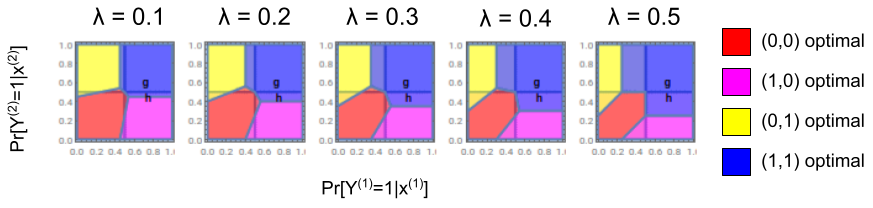}
    \label{fig:fnr_ls_lambda}
\end{minipage}
\hfill
\begin{minipage}{0.7\linewidth}
    \centering
    \includegraphics[width=\linewidth]{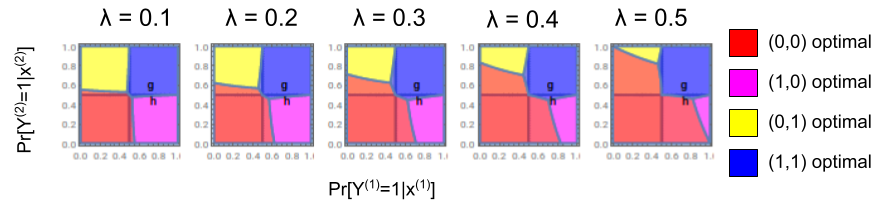}
    \label{fig:eeo_ls_lambda}
\end{minipage}
\caption{The level sets of different regularized properties as $\lambda$ changes. (Top to bottom: DP, FPR, FNR, EEO). The $g$ represents the ``average'' members of each group in the German lending dataset, and $h$ the heart disease risk dataset.} 
\label{fig:level_sets_lambda}
\end{figure}

\paragraph{Heart disease risk}
In the heart disease risk prediction dataset, we treat sex as the sensitive attribute, and observe $\Pr[Y=1 \mid S = 0] = 0.75$ and $\Pr[Y=1 \mid S = 1] = 0.449$ yields a $p_a, p_b$ pair warranting different treatments for the ``average'' member of each group, and $\Pr[S = 1] = 0.63$ for a more sensitive-attribute-balanced dataset.
The relationship between the optimal treatment of the ``average'' member of both groups as $\lambda$ changes can be seen in Figure~\ref{fig:level_sets_lambda}, denoted by $h$.

Figure~\ref{fig:violations_heart} shows the tradeoffs incurred by large weights on fairness violations, as accuracy of regularized losses tends to drop for $\lambda > 0.3$, which aligns with some of the improvements in fairness violations-- namely for demographic parity and false positive rates.
There is no significant difference in the FNR violation, regardless of $\lambda$, and an increase in the EEO violation; we conjecture this is due to numerical stability as the baseline EEO violation is very small.
In Figure~\ref{fig:violations_heart_dists} (R), this is supported by a higher range of EEO violations in the regularized models.

\begin{figure}
\begin{minipage}{0.19\linewidth}
    \centering
    \includegraphics[width=\linewidth]{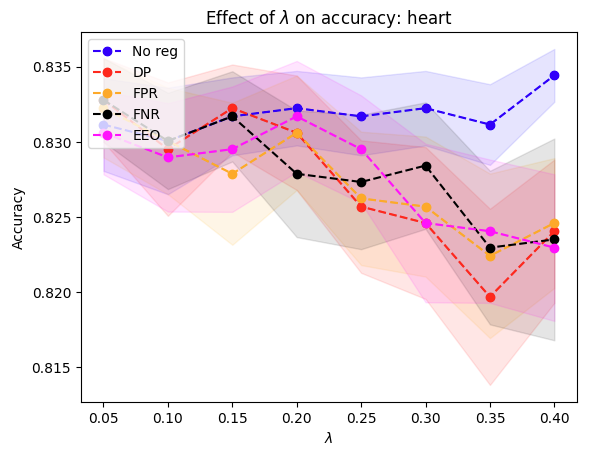}
    \label{fig:heart_acc}
\end{minipage}
\hfill
\begin{minipage}{0.19\linewidth}
    \centering
    \includegraphics[width=\linewidth]{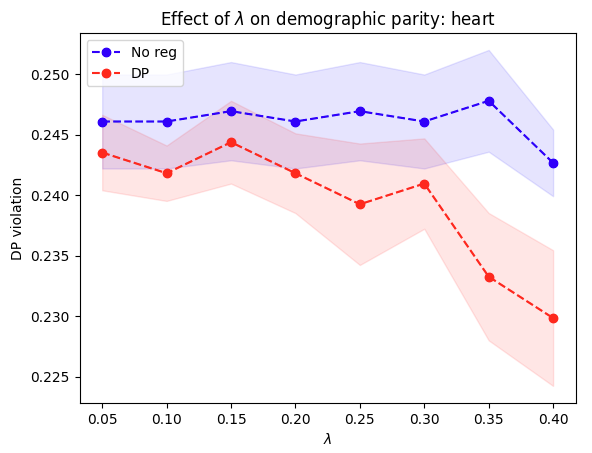}
    \label{fig:heart_DP}
\end{minipage}
\hfill
\begin{minipage}{0.19\linewidth}
    \centering
    \includegraphics[width=\linewidth]{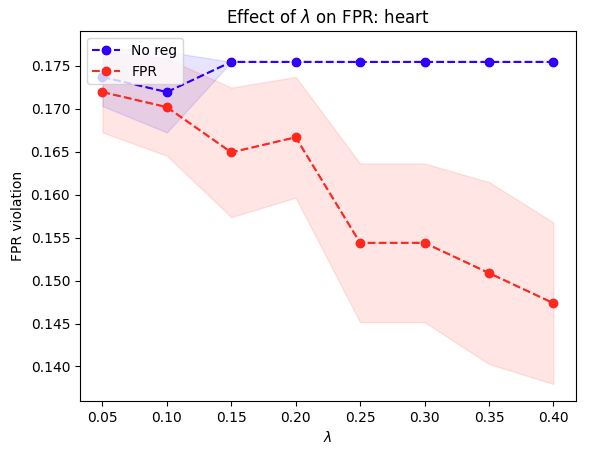}
    \label{fig:heart_FPR}
\end{minipage}
\hfill
\begin{minipage}{0.19\linewidth}
    \centering
    \includegraphics[width=\linewidth]{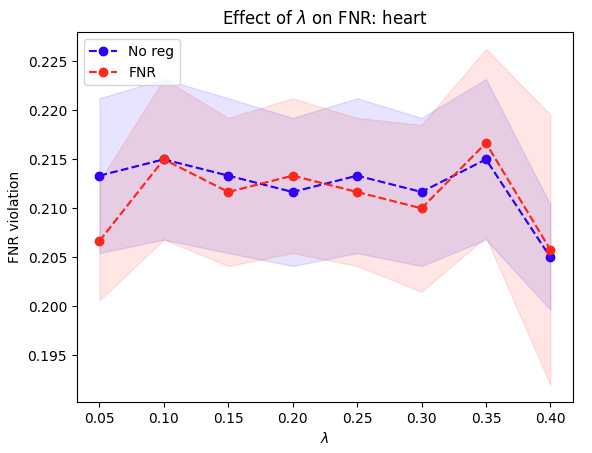}
    \label{fig:heart_FNR}
\end{minipage}
\hfill
\begin{minipage}{0.19\linewidth}
    \centering
    \includegraphics[width=\linewidth]{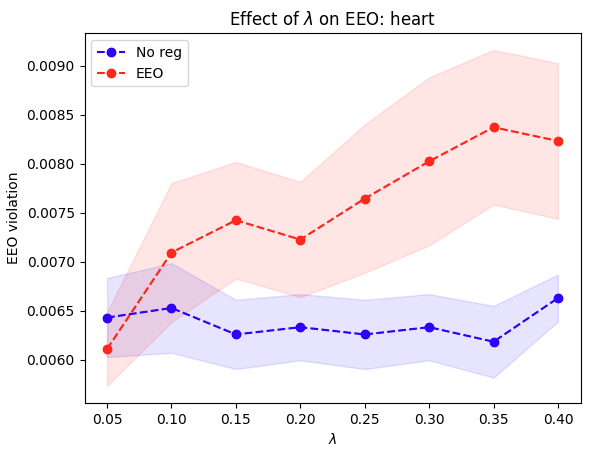}
    \label{fig:heart_EEO}
\end{minipage}
\caption{Effect of $\lambda$ on regularizer values on the heart disease risk dataset~\citep{heart}.}
\label{fig:violations_heart}
\end{figure}

\begin{figure}
\begin{minipage}{0.19\linewidth}
    \centering
    \includegraphics[width=\linewidth]{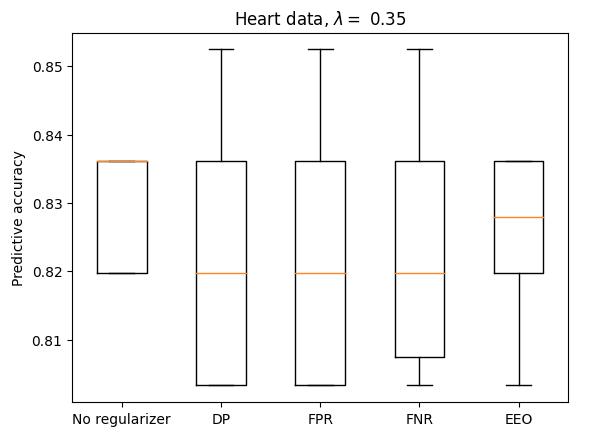}
    \label{fig:heart_acc_dist}
\end{minipage}
\hfill
\begin{minipage}{0.19\linewidth}
    \centering
    \includegraphics[width=\linewidth]{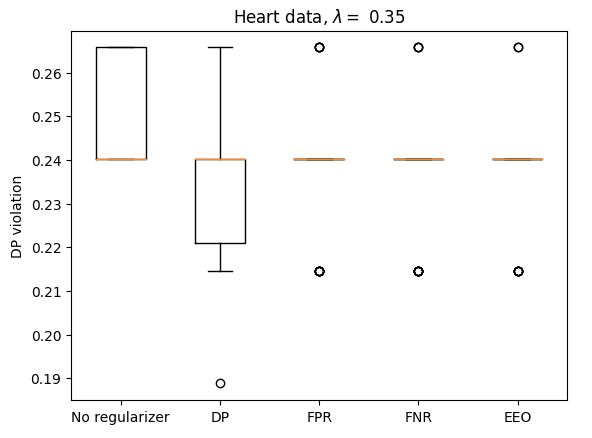}
    \label{fig:heart_DP_dist}
\end{minipage}
\hfill
\begin{minipage}{0.19\linewidth}
    \centering
    \includegraphics[width=\linewidth]{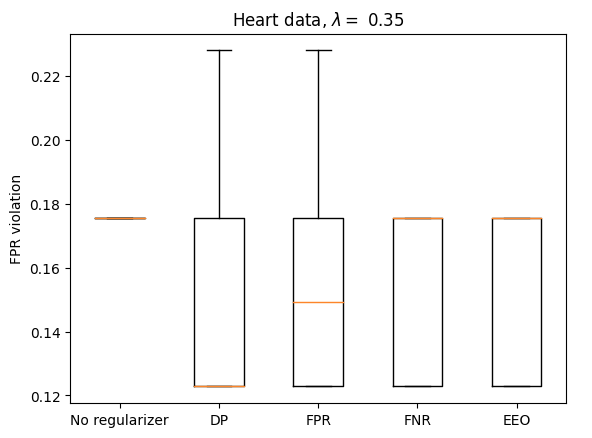}
    \label{fig:heart_FPR_dist}
\end{minipage}
\hfill
\begin{minipage}{0.19\linewidth}
    \centering
    \includegraphics[width=\linewidth]{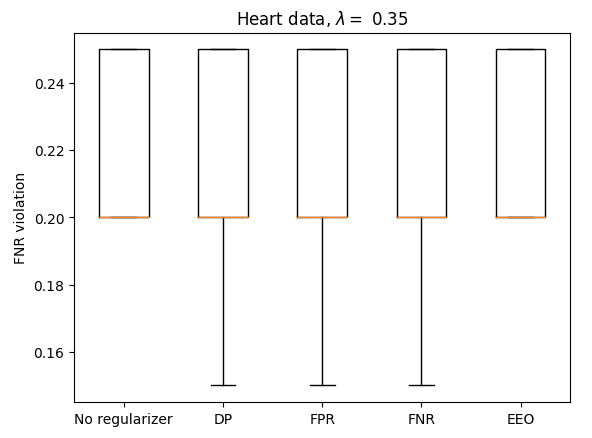}
    \label{fig:heart_FNR_dist}
\end{minipage}
\hfill
\begin{minipage}{0.19\linewidth}
    \centering
    \includegraphics[width=\linewidth]{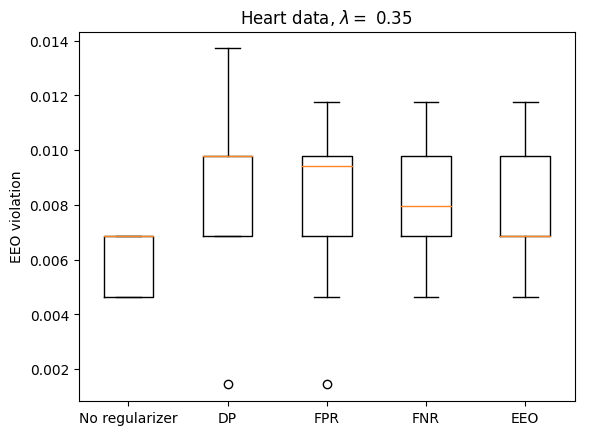}
    \label{fig:heart_EEO_dist}
\end{minipage}
\caption{Distributions of accuracy and fairness violations in heart disease data. In general, it seems the models are tending to make similar predictions, which often nearly equal medians.}
\label{fig:violations_heart_dists}
\end{figure}

\section{Discussion and Conclusion}
In this work, we extend the notion of property elicitation to consider regularized loss functions, and give a necessary and sufficient condition on a regularizer to be equivalent to the original property.
We apply this condition to demonstrate the (non-)equivalence of properties with a handful of regularizers common in the fair machine learning literature.
Finally, we show how the choice and weight of regularization function can change decision-making on synthetic data as well as the German lending and heart disease risk datasets.

\paragraph{Limitations and considerations}
The main intent of this work is to provide conceptual insight about how fairness regularizers change algorithmic decision-making and predictions.
The insights provided rely on the hypothesis class being sufficiently expressive, and should not be solely used to justify the use of a regularizer.
The addition of a regularizer and insights given are agnostic to the data itself and therefore agnostic to pre-processing and post-processing of data.
Additional pre- or post-processing of the data may change the elicited property, though we leave this to future work.

\paragraph{Future work}
There are many directions for future work.
This work serves as a proof of concept for the extension of property elicitation to accommodate regularization functions, demonstrated on a handful of regularizers, but applying the necessary and sufficient condition on the equivalence of properties under different regularizers and more general prediction tasks remains an open direction of work.
Moreover, it is important to understand how model complexity as well as pre- and post-processing of data can affect results. 
Finally, the addition of a regularizer seems to increase the complexity of the optimization problem linearly in $m$.
Understanding if there are more efficient ways to frame the optimization problem for certain regularization functions or data distributions also remains an open line of work.

\newpage
\subsection*{Acknowledgements}
This material is based upon work supported by the National Science Foundation under Award No. 2202898.
Thanks to Yiling Chen, Francisco Marmolejo Cossio, Esther Rolf, and Arpita Biswas for feedback and comments, as well as participants in the EC Gender Inclusion Workshop.
\bibliographystyle{abbrvnat}
\bibliography{refs}

\begin{thebibliography}{38}
\providecommand{\natexlab}[1]{#1}
\providecommand{\url}[1]{\texttt{#1}}
\expandafter\ifx\csname urlstyle\endcsname\relax
  \providecommand{\doi}[1]{doi: #1}\else
  \providecommand{\doi}{doi: \begingroup \urlstyle{rm}\Url}\fi

\bibitem[Agarwal et~al.(2019)Agarwal, Dudik, and Wu]{agarwal2019fair}
A.~Agarwal, M.~Dudik, and Z.~S. Wu.
\newblock Fair regression: Quantitative definitions and reduction-based
  algorithms.
\newblock In \emph{International Conference on Machine Learning}, pages
  120--129. PMLR, 2019.

\bibitem[Arutjothi and Senthamarai(2017)]{arutjothi2017prediction}
G.~Arutjothi and C.~Senthamarai.
\newblock Prediction of loan status in commercial bank using machine learning
  classifier.
\newblock In \emph{2017 International Conference on Intelligent Sustainable
  Systems (ICISS)}, pages 416--419. IEEE, 2017.

\bibitem[Bechavod and Ligett(2017)]{bechavod2017penalizing}
Y.~Bechavod and K.~Ligett.
\newblock Penalizing unfairness in binary classification.
\newblock \emph{arXiv preprint arXiv:1707.00044}, 2017.

\bibitem[Berk et~al.(2017)Berk, Heidari, Jabbari, Joseph, Kearns, Morgenstern,
  Neel, and Roth]{berk2017convex}
R.~Berk, H.~Heidari, S.~Jabbari, M.~Joseph, M.~Kearns, J.~Morgenstern, S.~Neel,
  and A.~Roth.
\newblock A convex framework for fair regression.
\newblock \emph{arXiv preprint arXiv:1706.02409}, 2017.

\bibitem[Blandin and Kash(2022)]{blandin2022fairness}
J.~Blandin and I.~Kash.
\newblock Fairness over utilities via multi-objective rewards.
\newblock 2022.

\bibitem[Brier et~al.(1950)]{brier1950verification}
G.~W. Brier et~al.
\newblock Verification of forecasts expressed in terms of probability.
\newblock \emph{Monthly weather review}, 78\penalty0 (1):\penalty0 1--3, 1950.

\bibitem[Denis et~al.(2021)Denis, Elie, Hebiri, and Hu]{denis2021fairness}
C.~Denis, R.~Elie, M.~Hebiri, and F.~Hu.
\newblock Fairness guarantee in multi-class classification.
\newblock \emph{arXiv preprint arXiv:2109.13642}, 2021.

\bibitem[Do et~al.(2022)Do, Putzel, Martin, Smyth, and Zhong]{do2022fair}
H.~Do, P.~Putzel, A.~S. Martin, P.~Smyth, and J.~Zhong.
\newblock Fair generalized linear models with a convex penalty.
\newblock In \emph{International Conference on Machine Learning}, pages
  5286--5308. PMLR, 2022.

\bibitem[Donini et~al.(2018)Donini, Oneto, Ben-David, Shawe-Taylor, and
  Pontil]{donini2018empirical}
M.~Donini, L.~Oneto, S.~Ben-David, J.~S. Shawe-Taylor, and M.~Pontil.
\newblock Empirical risk minimization under fairness constraints.
\newblock \emph{Advances in Neural Information Processing Systems}, 31, 2018.

\bibitem[Finocchiaro et~al.(2019)Finocchiaro, Frongillo, and
  Waggoner]{finocchiaro2022embedding}
J.~Finocchiaro, R.~Frongillo, and B.~Waggoner.
\newblock An embedding framework for consistent polyhedral surrogates, 2019.
\newblock URL \url{https://arxiv.org/abs/1907.07330}.

\bibitem[Fissler(2017)]{fissler2017higher}
T.~Fissler.
\newblock \emph{On higher order elicitability and some limit theorems on the
  Poisson and Wiener space}.
\newblock PhD thesis, 2017.

\bibitem[Frongillo and Kash(2014)]{frongillo2014general}
R.~Frongillo and I.~Kash.
\newblock General truthfulness characterizations via convex analysis.
\newblock In \emph{Web and {Internet} {Economics}}, pages 354--370. Springer,
  2014.

\bibitem[Frongillo and Kash(2019)]{frongillo2019general}
R.~M. Frongillo and I.~A. Kash.
\newblock General truthfulness characterizations via convex analysis, 2019.

\bibitem[Goel et~al.(2018)Goel, Yaghini, and Faltings]{goel2018non}
N.~Goel, M.~Yaghini, and B.~Faltings.
\newblock Non-discriminatory machine learning through convex fairness criteria.
\newblock In \emph{Proceedings of the AAAI Conference on Artificial
  Intelligence}, volume~32, 2018.

\bibitem[Goldstein et~al.(2017)Goldstein, Navar, and
  Carter]{goldstein2017moving}
B.~A. Goldstein, A.~M. Navar, and R.~E. Carter.
\newblock Moving beyond regression techniques in cardiovascular risk
  prediction: applying machine learning to address analytic challenges.
\newblock \emph{European heart journal}, 38\penalty0 (23):\penalty0 1805--1814,
  2017.

\bibitem[Hardt et~al.(2016)Hardt, Price, and Srebro]{hardt2016equality}
M.~Hardt, E.~Price, and N.~Srebro.
\newblock Equality of opportunity in supervised learning.
\newblock \emph{Advances in neural information processing systems}, 29, 2016.

\bibitem[Hebert-Johnson et~al.(2018)Hebert-Johnson, Kim, Reingold, and
  Rothblum]{hebert2018multicalibration}
U.~Hebert-Johnson, M.~Kim, O.~Reingold, and G.~Rothblum.
\newblock Multicalibration: Calibration for the
  ({C}omputationally-identifiable) masses.
\newblock In J.~Dy and A.~Krause, editors, \emph{Proceedings of the 35th
  International Conference on Machine Learning}, volume~80 of \emph{Proceedings
  of Machine Learning Research}, pages 1939--1948. PMLR, 10--15 Jul 2018.
\newblock URL \url{https://proceedings.mlr.press/v80/hebert-johnson18a.html}.

\bibitem[Huang and Vishnoi(2019)]{huang2019stable}
L.~Huang and N.~Vishnoi.
\newblock Stable and fair classification.
\newblock In K.~Chaudhuri and R.~Salakhutdinov, editors, \emph{Proceedings of
  the 36th International Conference on Machine Learning}, volume~97 of
  \emph{Proceedings of Machine Learning Research}, pages 2879--2890. PMLR,
  09--15 Jun 2019.
\newblock URL \url{https://proceedings.mlr.press/v97/huang19e.html}.

\bibitem[Jung et~al.(2020)Jung, Kannan, Lee, Pai, Roth, and
  Vohra]{jung2020fair}
C.~Jung, S.~Kannan, C.~Lee, M.~Pai, A.~Roth, and R.~Vohra.
\newblock Fair prediction with endogenous behavior.
\newblock In \emph{Proceedings of the 21st ACM Conference on Economics and
  Computation}, pages 677--678, 2020.

\bibitem[Jung et~al.(2021)Jung, Lee, Pai, Roth, and Vohra]{jung2021moment}
C.~Jung, C.~Lee, M.~Pai, A.~Roth, and R.~Vohra.
\newblock Moment multicalibration for uncertainty estimation.
\newblock In M.~Belkin and S.~Kpotufe, editors, \emph{Proceedings of Thirty
  Fourth Conference on Learning Theory}, volume 134 of \emph{Proceedings of
  Machine Learning Research}, pages 2634--2678. PMLR, 15--19 Aug 2021.
\newblock URL \url{https://proceedings.mlr.press/v134/jung21a.html}.

\bibitem[Kamiran and Calders(2009)]{kamiran2009classifying}
F.~Kamiran and T.~Calders.
\newblock Classifying without discriminating.
\newblock In \emph{2009 2nd international conference on computer, control and
  communication}, pages 1--6. IEEE, 2009.

\bibitem[Kamishima et~al.(2012)Kamishima, Akaho, Asoh, and
  Sakuma]{kamishima2012fairness}
T.~Kamishima, S.~Akaho, H.~Asoh, and J.~Sakuma.
\newblock Fairness-aware classifier with prejudice remover regularizer.
\newblock In \emph{Joint European conference on machine learning and knowledge
  discovery in databases}, pages 35--50. Springer, 2012.

\bibitem[Konstantinov and Lampert(2021)]{konstantinov2021fairness}
N.~Konstantinov and C.~H. Lampert.
\newblock Fairness through regularization for learning to rank.
\newblock \emph{CoRR}, abs/2102.05996, 2021.
\newblock URL \url{https://arxiv.org/abs/2102.05996}.

\bibitem[Kube et~al.(2023)Kube, Das, and Fowler]{kube2023community}
A.~R. Kube, S.~Das, and P.~J. Fowler.
\newblock Community- and data-driven homelessness prevention and service
  delivery: Optimizing for equity.
\newblock \emph{Journal of the American Medical Informatics Association},
  30\penalty0 (6):\penalty0 1032--1041, 04 2023.
\newblock ISSN 1527-974X.
\newblock \doi{10.1093/jamia/ocad052}.

\bibitem[Lambert(2018)]{lambert2018elicitation}
N.~S. Lambert.
\newblock Elicitation and evaluation of statistical forecasts.
\newblock 2018.
\newblock URL
  \url{https://web.stanford.edu/\~nlambert/papers/elicitability.pdf}.

\bibitem[Lambert and Shoham(2009)]{lambert2009eliciting}
N.~S. Lambert and Y.~Shoham.
\newblock Eliciting truthful answers to multiple-choice questions.
\newblock In \emph{Proceedings of the 10th {ACM} conference on {Electronic}
  commerce}, pages 109--118, 2009.

\bibitem[Lambert et~al.(2008)Lambert, Pennock, and
  Shoham]{lambert2008eliciting}
N.~S. Lambert, D.~M. Pennock, and Y.~Shoham.
\newblock Eliciting properties of probability distributions.
\newblock In \emph{Proceedings of the 9th {ACM} {Conference} on {Electronic}
  {Commerce}}, pages 129--138, 2008.

\bibitem[Lloyd-Jones(2010)]{lloyd2010cardiovascular}
D.~M. Lloyd-Jones.
\newblock Cardiovascular risk prediction: basic concepts, current status, and
  future directions.
\newblock \emph{Circulation}, 121\penalty0 (15):\penalty0 1768--1777, 2010.

\bibitem[Mireshghallah et~al.(2021)Mireshghallah, Inan, Hasegawa, Rühle,
  Berg-Kirkpatrick, and Sim]{mireshghallah2021privacy}
F.~Mireshghallah, H.~A. Inan, M.~Hasegawa, V.~Rühle, T.~Berg-Kirkpatrick, and
  R.~Sim.
\newblock Privacy regularization: Joint privacy-utility optimization in
  language models, 2021.

\bibitem[Noarov and Roth(2023)]{noarov2023statistical}
G.~Noarov and A.~Roth.
\newblock The statistical scope of multicalibration.
\newblock In A.~Krause, E.~Brunskill, K.~Cho, B.~Engelhardt, S.~Sabato, and
  J.~Scarlett, editors, \emph{Proceedings of the 40th International Conference
  on Machine Learning}, volume 202 of \emph{Proceedings of Machine Learning
  Research}, pages 26283--26310. PMLR, 23--29 Jul 2023.
\newblock URL \url{https://proceedings.mlr.press/v202/noarov23a.html}.

\bibitem[Pleiss et~al.(2017)Pleiss, Raghavan, Wu, Kleinberg, and
  Weinberger]{pleiss2017fairness}
G.~Pleiss, M.~Raghavan, F.~Wu, J.~Kleinberg, and K.~Q. Weinberger.
\newblock On fairness and calibration.
\newblock \emph{Advances in neural information processing systems}, 30, 2017.

\bibitem[Rahman(2021)]{heart}
R.~Rahman.
\newblock Heart attack analysis and prediction dataset.
\newblock
  https://www.kaggle.com/datasets/rashikrahmanpritom/heart-attack-analysis-prediction-dataset,
  2021.
\newblock URL
  \url{https://www.kaggle.com/datasets/rashikrahmanpritom/heart-attack-analysis-prediction-dataset}.

\bibitem[Savage(1971)]{savage1971elicitation}
L.~J. Savage.
\newblock Elicitation of personal probabilities and expectations.
\newblock \emph{Journal of the American Statistical Association}, 66\penalty0
  (336):\penalty0 783--801, 1971.

\bibitem[Sheikh et~al.(2020)Sheikh, Goel, and Kumar]{sheikh2020approach}
M.~A. Sheikh, A.~K. Goel, and T.~Kumar.
\newblock An approach for prediction of loan approval using machine learning
  algorithm.
\newblock In \emph{2020 International Conference on Electronics and Sustainable
  Communication Systems (ICESC)}, pages 490--494. IEEE, 2020.

\bibitem[Singh et~al.(2021)Singh, Yadav, Awasthi, and
  Partheeban]{singh2021prediction}
V.~Singh, A.~Yadav, R.~Awasthi, and G.~N. Partheeban.
\newblock Prediction of modernized loan approval system based on machine
  learning approach.
\newblock In \emph{2021 International Conference on Intelligent Technologies
  (CONIT)}, pages 1--4. IEEE, 2021.

\bibitem[Steinwart et~al.(2014)Steinwart, Pasin, Williamson, and
  Zhang]{steinwart2014elicitation}
I.~Steinwart, C.~Pasin, R.~Williamson, and S.~Zhang.
\newblock Elicitation and {Identification} of {Properties}.
\newblock In \emph{Proceedings of {The} 27th {Conference} on {Learning}
  {Theory}}, pages 482--526, 2014.

\bibitem[Williamson and Menon(2019)]{williamson2019fairness}
R.~Williamson and A.~Menon.
\newblock Fairness risk measures.
\newblock In \emph{International Conference on Machine Learning}, pages
  6786--6797. PMLR, 2019.

\bibitem[Zafar et~al.(2017)Zafar, Valera, Rogriguez, and
  Gummadi]{zafar2017fairness}
M.~B. Zafar, I.~Valera, M.~G. Rogriguez, and K.~P. Gummadi.
\newblock Fairness constraints: Mechanisms for fair classification.
\newblock In \emph{Artificial intelligence and statistics}, pages 962--970.
  PMLR, 2017.

\end{thebibliography}

\newpage
\appendix

\section{Additional example of non-equivalent properties}\label{app:omitted-regs}

\subsection{Equalized FNR}\label{subsec:fnr}
Similarly, we consider false negative rates.
Our objective is
\begin{align*}
    L^{FNR,\lambda}(\vec t; \vec s; \vec p) 
    &= \frac 1 m \sum_i L(t^\i, p^\i) + \lambda \left| \frac 1 {n_a} \sum_{i : s^\i = a, t^\i = 0} p^\i - \frac 1 {n_b} \sum_{i : s^\i = b, t^\i = 0} p^\i \right|
\end{align*}

Like the FPR regularizer, since  the FNR regularizer computes the difference of false negative rates between groups, one can observe that a way to reduce the false negative rate of a group is to assign more positive treatments $t^\i = 1$.
Again, we see in figure~\ref{fig:fnr_example} that the FNR regularizer then makes it worse for an algorithm to assign the negative treatment to an agent $i$ even if $p^\i$ slightly less than $1/2$. 

\begin{restatable}{corollary}{fnr}\label{cor:fnr}
Let $L : [0,1] \times \{0,1\} \to [0,1]$ and $\psi : r \mapsto \Ind{r \geq 1/2}$ indirectly elicit the mode over $\Y = \{0,1\}$ such that $L(y,y) = 0$, and let $\Theta^{FPR,\lambda} := \psi \circ \prop{L^{FPR,\lambda}}$.
Then $\Theta^{FPR, \lambda}$ is not equivalent to the mode for $\lambda > 0$.
\end{restatable}

\begin{figure}
\includegraphics[width=\linewidth]{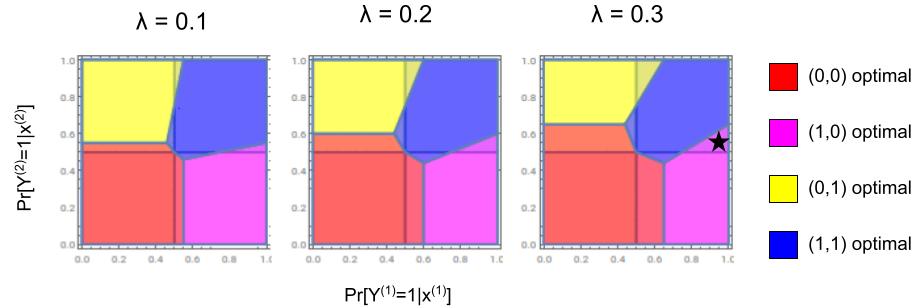}
\caption{Visualizing the level sets of the $FNR$-regularized property $\Theta^{FNR,\lambda}$ for different values of $\lambda \in [0,1]$, where $m = 2$ and $s = (a,b)$. Each point $(p^{(1)}, p^{(2)})$ in a square represents $(\Pr_{p^{(1)}}[Y=1], \Pr_{p^{(2)}}[Y=1])$, and each colored cell represents sets of $(p^{(1)}, p^{(2)})$ pairs such that the optimal treatment is the same for all points in the cell. For example, the magenta cell is the set of distributions where the decision-maker prefers to attribute the positive treatment ($t^\i = 1$) to the agent in group $a$, and the negative treatment ($t^\i = 0$) to the agent in group $b$.}
\label{fig:fnr_example}
\end{figure}

\end{document}